\documentclass[11pt]{article}

\usepackage[margin=1in]{geometry}

\usepackage[english]{babel}

\usepackage{amsmath,amsthm,amssymb}
\usepackage{mathtools}
\usepackage{mathrsfs}
\usepackage{amsfonts}
\usepackage{bm}

\usepackage{graphicx}
\usepackage{booktabs}
\usepackage{caption}
\usepackage{subcaption}

\usepackage{enumitem}

\usepackage{algorithm}
\usepackage{algorithmic}

\usepackage[hidelinks]{hyperref}

\newcommand{\R}{\mathbb{R}}

\newcommand{\E}{\mathbb{E}}
\newcommand{\Prob}{\mathbb{P}}

\newcommand{\cE}{\mathcal{E}}
\newcommand{\cH}{\mathcal{H}}
\newcommand{\cK}{\mathcal{K}}
\newcommand{\cL}{\mathcal{L}}
\newcommand{\cN}{\mathcal{N}}
\newcommand{\Om}{\Omega}
\newcommand{\Rd}{\mathbb{R}^d}
\newcommand{\RD}{\mathbb{R}^D}

\newcommand{\diag}{\operatorname{diag}}

\newcommand{\rng}{\operatorname{rng}}
\newcommand{\rnk}{\operatorname{rnk}}

\newcommand{\tr}{\operatorname{Tr}}
\newcommand{\id}{\operatorname{id}}

\newcommand{\norm}[1]{\lVert#1\rVert}
\newcommand{\abs}[1]{|#1|}

\DeclarePairedDelimiterX{\frobinner}[2]{\langle}{\rangle_F}{#1\, ,\, #2}

\newcommand{\pder}[2][]{\frac{\partial#1}{\partial#2}}

\newcommand{\sigmin}{\sigma_{\min}}
\newcommand{\sigmax}{\sigma_{\max}}

\newtheorem{thm}{Theorem}[section]
\newtheorem{lem}[thm]{Lemma}
\newtheorem{prop}[thm]{Proposition}
\newtheorem{cor}[thm]{Corollary}

\newtheorem{assumption}[thm]{Assumption}

\title{Geometric regularization of autoencoders via observed stochastic dynamics}
\author{
Sean Hill\thanks{Department of Mathematics and Statistics, University at Albany, SUNY (\texttt{smhill@albany.edu}).}
\and
Felix X.-F.\ Ye\thanks{Department of Mathematics and Statistics, University at Albany, SUNY (\texttt{xye2@albany.edu}).}
}
\date{}

\begin{document}
\sloppy
\maketitle

\begin{abstract}
Stochastic dynamical systems with slow or metastable
behavior evolve, on long time scales, on an unknown
low-dimensional manifold in high-dimensional ambient
space.
Building a reduced simulator from short-burst ambient
ensembles is a long-standing problem: local-chart
methods like ATLAS suffer from exponential landmark
scaling and per-step reprojection, while autoencoder
alternatives leave tangent-bundle geometry poorly
constrained, and the errors propagate into the
learned drift and diffusion.
We observe that the ambient covariance~$\Lambda$
already encodes coordinate-invariant tangent-space
information, its range spanning the tangent bundle.
Using this, we construct a tangent-bundle penalty and
an inverse-consistency penalty for a three-stage
pipeline (chart learning, latent drift, latent
diffusion) that learns a single nonlinear chart and
the latent SDE.
The penalties induce a function-space metric, the
$\rho$-metric, strictly weaker than the Sobolev $H^1$
norm yet achieving the same chart-quality
generalization rate up to logarithmic factors.
For the drift, we derive an encoder-pullback target
via It\^o's formula on the learned encoder and prove
a bias decomposition showing the standard
decoder-side formula carries systematic error for any
imperfect chart.
Under a $W^{2,\infty}$ chart-convergence assumption,
chart-level error propagates controllably to weak
convergence of the ambient dynamics and to
convergence of radial mean first-passage times.
Experiments on four surfaces embedded in up to $201$
ambient dimensions reduce radial MFPT error by
$50$--$70\%$ under rotation dynamics and achieve the
lowest inter-well MFPT error on most
surface--transition pairs under metastable
M\"uller--Brown Langevin dynamics, while reducing
end-to-end ambient coefficient errors by up to an
order of magnitude relative to an unregularized
autoencoder.
\end{abstract}




\section{Introduction}\label{sec:intro}

Many stochastic systems of scientific interest evolve, on long time
scales, on an unknown low-dimensional invariant manifold
$M\subset\R^D$~\cite{chirikjian,Hsu,PA_book08,NB_book06}.
The effective dynamics on~$M$ arise, for example, after averaging or
homogenization of a multiscale (fast--slow)
SDE~\cite{leimkuhler2015molecular,PA_book08}, and the practical goal is
to learn a reduced simulator that faithfully reproduces the slow
dynamics, including drift, diffusion, and passage-time statistics, directly from
data.
Because the full system is stiff, long direct simulations are prohibitively expensive.
A data-driven alternative is to run \emph{short-burst ensembles}
from strategically chosen initial conditions: each burst is
cheap, embarrassingly parallel, and reveals local effective dynamics
and local geometry at the initial point~\cite{Miles17,yym}.
The ATLAS framework~\cite{Miles17,yym} organizes this idea into a
complete simulator.
A well-distributed network of landmarks~$\{x_i\}_{i=1}^m$ is placed
on~$M$, and short-burst ensembles at each landmark yield local
estimates of the ambient drift and covariance
$\{b(x_i),\Lambda(x_i)\}_{i=1}^m$ with statistical
guarantees~\cite{Miles17}.
A Gaussian-kernel smoother then interpolates these estimates into
continuous fields, and the simulator steps in ambient coordinates using
linear tangent-plane models at the landmarks.
However, three additional computational bottlenecks arise in the ATLAS simulator:
\begin{enumerate}[leftmargin=*,nosep]
\item \textbf{Curse of dimensionality.}
  The number of landmarks~$m$ and the number of neighbors used in
  kernel smoothing both scale \emph{exponentially} with the manifold
  dimension~$d$.
  The root cause is the reliance on linear local models, namely tangent
  ellipsoids glued together to cover~$M$.
\item \textbf{Re-projection onto~$M$.}
  Because ATLAS simulates in ambient $\R^D$, every time step requires
  re-projection onto~$M$ via weighted oblique affine projections at
  neighboring landmarks; without this, sample paths quickly drift off
  the manifold due to accumulation of linearization error.
\item \textbf{Per-step simulation cost.}
  Each step blends $D\times D$ covariance matrices from neighboring
  landmarks at $O(m D^2)$ cost;
  in our experiments ATLAS is already computationally infeasible at
  $D{=}201$ with only $m{=}200$ landmarks and $d{=}2$.
\end{enumerate}
Autoencoders~\cite{lee2020model,otto2019linearly,peterfreund2020local,
floryan2022data,champion2019data,kevrekidis2024thinner,
schonsheck2020chartautoencodersmanifoldstructured}
address all three ATLAS bottlenecks:
an encoder~$\pi\!:\R^D\!\to\R^d$ and
decoder~$\phi\!:\R^d\!\to\R^D$
replace exponentially many linear patches with a single nonlinear
learned chart, and simulating intrinsically in $d$~dimensions
eliminates re-projection entirely.
Several lines of work build on this idea.
Autoencoder-based methods have been combined with
deterministic latent
dynamics~\cite{champion2019data,lusch2018deep,linot2020deep,
floryan2022data,bertalan2019learning}
and, more recently, with latent
SDEs~\cite{hasan2022identifying,li2020scalable},
but none use observed dynamics to regularize
the chart geometry; the contractive
autoencoder~\cite{contractiveAutoEncoders} and the
pullback metric of~\cite{arvanitidis2018latent}
regularize geometry but without a dynamics-derived
tangent target.
Outside the autoencoder framework,
the equation-free approach~\cite{kevrekidis2003equation}
learns SDE coefficients in externally supplied reduction
coordinates, and diffusion-map
methods~\cite{coifman2008diffusion,evangelou2023double,dietrich2023learning}
estimate generator eigenfunctions from dense trajectory
data in a spectral regime.
Our pipeline instead targets the sparse-landmark setting
($N{=}50$) and requires a decoder for ambient-space
reconstruction; neither the equation-free nor the
spectral setting is directly comparable.
When the manifold is \emph{known},
Riemannian diffusion
models~\cite{debortoli,huang2022riemannian} and
normalizing
flows~\cite{mathieu2020riemannian} define dynamics
intrinsically.
No existing method jointly learns a geometrically
regularized chart and SDE coefficients from sparse
ambient observations.

Training with reconstruction loss alone, however,
controls only the $L^2$ error of the chart map, leaving the
first-order geometry of tangent spaces poorly
constrained~\cite{floryan2022data}, especially in the
sparse-data regime where only tens of observations are
available.
The resulting geometric errors amplify when the chart
is used to simulate dynamics: the It\^o correction
depends on both the decoder Jacobian and
Hessian~\cite{armstrong2018intrinsicjets,Armstrong2024},
so tangent-bundle and curvature inaccuracies propagate
directly into the recovered drift and diffusion, and
the chart cannot extrapolate reliably beyond its
training set.
Existing autoencoder
theory~\cite{schonsheck2020chartautoencodersmanifoldstructured,
liu2023deepnonparametricestimationintrinsic} achieves a squared
generalization error of order $m^{-2/(d+2)}\log^4 m$ from point-cloud
data but does not address this geometric gap.
Sobolev training yields
provably better generalization~\cite{yang2025deeperwiderperspectiveoptimal,
yang2025deepneuralnetworksgeneral}, but requires
chart-derivative labels;
semi-supervised
methods~\cite{schonsheck2024semisupervisedmanifoldlearningcomplexity}
improve chart quality when function values are observed,
but treating drift and covariance as regression
targets in latent coordinates yields chart-dependent
quantities.
Neither is directly applicable.

Our framework addresses this geometric gap by
exploiting a key observation: the ambient
covariance~$\Lambda$ already encodes coordinate-invariant
geometric information.
Its range spans the tangent
space~\cite{coifman2006diffusion}, and the resulting tangent
projector $P=D\phi\,g^{-1}D\phi^\top$ is invariant under
reparameterization.
Penalizing the discrepancy between~$P$ and the data-derived
projector regularizes first-order geometry without
chart-derivative labels and defines a function-space
metric~$\rho$ strictly between~$L^2$ and the
Sobolev~$H^1$ norm.
For the latent drift, It\^o's formula applied to the learned
encoder yields an exact pullback target, avoiding the
systematic bias of the decoder-side formula.
This paper develops the single-chart theory for this
program; extending to a multi-chart atlas is a natural
next step discussed in Section~\ref{sec:discussion}.

\noindent \textbf{Contributions.}
\begin{enumerate}
\item A three-stage pipeline whose geometric losses
  (tangent-bundle, drift, diffusion) are coordinate-invariant:
  it extracts geometric penalties from~$\Lambda$, fits the
  latent drift via an encoder-pullback target, and
  learns the diffusion under metric-weighted losses.
\item A generalization guarantee showing that an idealized
  $\rho$-ERM achieves the same rate as Sobolev $H^1$
  training for chart quality, with controlled propagation
  to SDE coefficients and weak convergence
  (\S\ref{sec:generalization}).

\end{enumerate}

\noindent \textbf{Outline.}
Section~\ref{sec:geometry} derives geometric conditions and the
encoder-pullback drift target from SDE coefficients.
Section~\ref{sec:rho-metric} introduces the $\rho$-metric and
geometric penalties.
Section~\ref{sec:generalization} establishes generalization bounds,
the decoder-side bias analysis, and error propagation to dynamics.
Section~\ref{sec:experiments} reports experiments;
Section~\ref{sec:discussion} discusses extensions.

\section{Geometry from Stochastic Dynamics}\label{sec:geometry}

We work in the \emph{single-chart} setting throughout:
the manifold~$M$ is assumed to be covered by one
coordinate chart.
This restricts the theory to manifolds diffeomorphic to
a bounded open subset of~$\R^d$, but the penalties and
pipeline extend chart-wise to an atlas
(Section~\ref{sec:discussion}).
Let $M$ be a smooth, connected, $d$-dimensional Riemannian
submanifold of~$\R^D$ ($d<D$), parameterized by a smooth embedding
$\phi:\Om\to\R^D$ (the \emph{decoder}) from a fixed open,
bounded domain $\Om\subset\R^d$, with encoder
$\pi=\phi^{-1}:M\to \Om$.
We write $D\phi$ for the Jacobian,
$g=D\phi^\top D\phi$ for the induced metric, and
$P=D\phi\,g^{-1}D\phi^\top$ for the orthogonal projector onto
$T_{\phi(z)}M$.

\subsection{It\^o transformation rules}\label{ssec:background}

Expressed in It\^o form in ambient coordinates, the effective dynamics
on~$M$ take the form
$$dX_t = b(X_t)\,dt + \eta(X_t)\,dW_t, \quad X_t\in M\subset\R^D,$$
where $b$ is the effective drift, $\eta$ is the ambient diffusion
coefficient, and $\Lambda=\eta\eta^\top$ is the rank-$d$
covariance.
Within a chart $(\Om,\phi)$, the local-coordinate SDE takes the form
$$dZ_t = \mu(Z_t)\,dt + \sigma(Z_t)\,dB_t, \quad Z_t\in \Om\subset\R^d.$$
We write $\Sigma=\sigma\sigma^\top$ for the local covariance and
define the It\^o correction vector
$q^i(\Sigma) := \frobinner{\Sigma}{\nabla^2\phi^i}$ for
$i=1,\dots,D$.

\begin{lem}[Local to ambient]\label{lem:local-to-ambient}
  The ambient drift and covariance are
  \[
    b = D\phi\,\mu + \tfrac{1}{2}\,q(\Sigma),
    \qquad
    \Lambda = D\phi\,\Sigma\,D\phi^\top.
  \]
\end{lem}

\begin{proof}
  Apply It\^o's formula to each component $X^i = \phi^i(Z)$.
  The first-order term gives $D\phi^i\,\mu$, and the second-order
  quadratic-variation term gives
  $\frac{1}{2}\frobinner{\Sigma}{\nabla^2\phi^i}$.
  Stacking over $i=1,\dots,D$ yields the ambient drift~$b$.
  The diffusion coefficient transforms by the chain rule as
  $D\phi\,\sigma$, so the covariance is
  $\Lambda = D\phi\,\Sigma\,D\phi^\top$.
\end{proof}

To invert this transformation and recover the latent drift~$\mu$ from
the ambient data $(b,\Lambda)$, we need an identity that relates the
It\^o correction in the two charts.

\begin{lem}[Quadratic covariation under chart change]\label{lem:qcov-change}
  With $\phi,\pi,\Lambda,\Sigma$ as above, assume
  $\rng(\Lambda(x))\subset T_xM$ for all $x\in M$.
  Then
  $\frobinner{\Lambda}{\nabla^2\pi}
    = -D\pi\,\frobinner{\Sigma}{\nabla^2\phi}$.
\end{lem}

\begin{proof}
  Differentiate the identity $\pi\circ\phi=\id_\Om$.
  At first order, $D\pi\,D\phi=I_d$.
  Differentiating once more gives
  $\pder[^2\pi^i]{x^k\partial x^l}\,
  \pder[\phi^l]{z^m}\,\pder[\phi^k]{z^j}
  + \pder[\pi^i]{x^k}\,\pder[^2\phi^k]{z^j\partial z^m} = 0$.
  Contracting both sides with a $d\times d$ PSD matrix $S_{mj}$
  and recognizing the Frobenius inner products gives
  $\frobinner{D\phi\,S\,D\phi^\top}{\nabla^2\pi^i}
  = -\bigl(D\pi\,\frobinner{S}{\nabla^2\phi}\bigr)_i$.
  Stacking over $i$ gives the identity for any PSD~$S$.
  Tangentiality ($\rng\Lambda\subset T_xM$) ensures
  $\Lambda=D\phi\,\Sigma\,D\phi^\top$ with
  $\Sigma=D\pi\,\Lambda\,D\pi^\top\in S_+^d$
  (PSD since $\Sigma=\sigma\sigma^\top$);
  $S=\Sigma$ is therefore an admissible choice and
  recovers the stated identity.
\end{proof}

\begin{lem}[Ambient to local]\label{lem:ambient-to-local}
  With $\phi,\pi,b,\Lambda,q$ as above, assume
  $\rng(\Lambda(x))\subset T_xM$ for all $x\in M$.
  Then $Z=\pi(X)$ solves a local SDE with
  covariance and drift 
  $$\Sigma = D\pi\,\Lambda\,D\pi^\top, \qquad \mu = D\pi\Bigl[b - \tfrac{1}{2}\,q(\Sigma)\Bigr].$$
\end{lem}

\begin{proof}
  Apply It\^o's formula to $Z=\pi(X)$.
  The drift of $Z$ is
  $\mu = D\pi\,b + \tfrac{1}{2}\,\frobinner{\Lambda}{\nabla^2\pi}$
  and the local covariance is
  $\Sigma = D\pi\,\Lambda\,D\pi^\top$.
  Applying Lemma~\ref{lem:qcov-change} to the Hessian correction
  yields
  $\mu = D\pi\,b + \tfrac{1}{2}\,\frobinner{\Lambda}{\nabla^2\pi}
       = D\pi\,b - \tfrac{1}{2}\,D\pi\,\frobinner{\Sigma}{\nabla^2\phi}
       = D\pi\bigl[b - \tfrac{1}{2}\,\frobinner{\Sigma}{\nabla^2\phi}\bigr]$,
  as claimed.
\end{proof}

\noindent\textbf{Computing the latent drift target.}\label{ssec:enc-pull}
The local-coordinate coefficients of
Lemma~\ref{lem:ambient-to-local} coincide with
the optimal mean-square approximation in the It\^o-jet
framework~\cite{armstrong2018intrinsicjets,Armstrong2024}.
The lemma gives two equivalent representations
of the latent drift~$\mu$: a \emph{decoder-side} form
$\mu = D\pi[b - \tfrac{1}{2}q(\Sigma)]$ that requires the
decoder Hessian via~$q(\Sigma)$, and an \emph{encoder-side} form
from the intermediate step of the proof that uses only the
encoder derivatives:
\begin{equation}\label{eq:enc-pull}
  \mu
  = D\pi\,b + \tfrac{1}{2}\,\frobinner{\Lambda}{\nabla^2\pi}.
\end{equation}
For a learned encoder~$\pi_\theta$, the Jacobian $D\pi_\theta$
and Hessian $\nabla^2\pi_\theta$ are available via automatic
differentiation, so the latent drift target
\begin{equation}\label{eq:enc-pull-target}
  \mu(z_i)
  = D\pi_\theta(x_i)\,b(x_i)
    + \tfrac{1}{2}\,\frobinner{\Lambda(x_i)}{\nabla^2\pi_\theta(x_i)}
\end{equation}
is computable without touching the decoder.
This \emph{encoder-pullback drift} trains the latent drift
network in Stage~2 (Algorithm~\ref{alg:pipeline}).
For an imperfect autoencoder the two formulas diverge: the
decoder-side target carries a deterministic bias for any fixed
learned chart
(Proposition~\ref{prop:dec-bias}), making the choice between them
consequential for trajectory fidelity.

\subsection{Geometric consistency and coordinate invariance}\label{ssec:tangent-from-cov}

The Itô rules above show how to convert between ambient and latent
SDE coefficients given a chart.
We now turn to the question of what the ambient covariance~$\Lambda$
reveals about the \emph{geometry} of~$M$ itself.
The relation $\Lambda = D\phi\,\Sigma\,D\phi^\top$
(Lemma~\ref{lem:local-to-ambient}) implies that the
range of~$\Lambda$ must span the tangent space defined by~$\phi$;
we call this the \emph{geometric consistency} condition~(\textbf{GC}).

Since $\Sigma\in S_{++}^d$ and $D\phi$ has rank~$d$,
$\Lambda$ has rank~$d$ with
$\rng(\Lambda) = \rng(D\phi) = T_xM$.
Writing
$\Lambda = Q\,\diag(\lambda_1,\dots,\lambda_D)\,Q^\top$
with $\lambda_1\ge\cdots\ge\lambda_d > 0
= \lambda_{d+1}=\cdots=\lambda_D$,
the top-$d$ eigenvectors
$U_d := Q_{1:d}\in\R^{D\times d}$ satisfy
$U_d^\top U_d = I_d$ and
$P = U_d U_d^\top$ is the orthogonal projection onto
$T_xM$.
Thus the tangent projector $P(x)$ can be recovered from
the ambient covariance~$\Lambda(x)$ by spectral
truncation, connecting observable dynamics data to the
tangent bundle of~$M$.
For~(GC) to serve as a training penalty, it must
not depend on the parameterization of~$\Om$.

\begin{lem}[Coordinate invariance]\label{lem:coord-invariance}
  Let $\alpha:\tilde{\Om}\to\Om$ be a diffeomorphism and
  $\tilde\phi = \phi\circ\alpha$.
  \textup{(i)}~The orthogonal projector
  $P=D\phi\,g^{-1}D\phi^\top$ is invariant:
  $\tilde P = P$.
  \textup{(ii)}~The It\^o drift~$\mu$ is not a geometric vector,
  but the difference of any two It\^o drifts sharing a common
  local covariance~$\Sigma$ transforms as a contravariant vector:
  $\tilde\mu_1 - \tilde\mu_2 = D\alpha^{-1}(\mu_1-\mu_2)$.
  In particular,
  $\norm{\mu_1-\mu_2}_{\tilde g}^2
  = \norm{\mu_1-\mu_2}_g^2$.
  \textup{(iii)}~For any symmetric $(2{,}0)$-tensor difference
  $\Delta\Sigma = \Sigma_1 - \Sigma_2$,
  the metric-weighted norm
  $\tr\bigl((g\,\Delta\Sigma)^2\bigr)$ is invariant.
\end{lem}

\begin{proof}
  \textbf{(i)}
  The chain rule gives
  $D[\tilde{\phi}]=D\phi\,D\alpha$, so
  $\tilde{g}=D\alpha^\top g\,D\alpha$ and
  $\tilde{P}
  =D[\tilde{\phi}]\,\tilde{g}^{-1}D[\tilde{\phi}]^\top
  =D\phi\,g^{-1}D\phi^\top=P$.

  \textbf{(ii)}
  Under $\alpha$, the It\^o drift acquires a Hessian correction
  $\tilde\mu = D\alpha^{-1}\mu
  + \tfrac{1}{2}\frobinner{\Sigma}{\nabla^2\alpha^{-1}}$.
  Both $\mu_1$ and $\mu_2$ receive the same correction, so
  $\tilde\mu_1-\tilde\mu_2 = D\alpha^{-1}(\mu_1-\mu_2)$.
  Then
  $\norm{\tilde\mu_1-\tilde\mu_2}_{\tilde g}^2
  = (\mu_1{-}\mu_2)^\top (D\alpha)^{-\top}D\alpha^\top g\,D\alpha\,D\alpha^{-1}(\mu_1{-}\mu_2)
  = \norm{\mu_1-\mu_2}_g^2$.

  \textbf{(iii)}
  The covariance difference transforms as
  $\widetilde{\Delta\Sigma} = D\alpha^{-1}\Delta\Sigma\,(D\alpha)^{-\top}$
  and the metric as $\tilde g = D\alpha^\top g\,D\alpha$.
  Multiplying:
  $\tilde g\,\widetilde{\Delta\Sigma}
  = D\alpha^\top g\,\Delta\Sigma\,(D\alpha)^{-\top}$.
  Squaring and taking the trace, the similarity factors cancel:
  $\tr\bigl((\tilde g\,\widetilde{\Delta\Sigma})^2\bigr)
  = \tr\bigl(D\alpha^\top(g\,\Delta\Sigma)^2 (D\alpha)^{-\top}\bigr)
  = \tr\bigl((g\,\Delta\Sigma)^2\bigr)$.
\end{proof}

Consequently, both~(GC) and the tangent-bundle penalty derived from
it in Section~\ref{sec:rho-metric} are chart-invariant, and the
metric-weighted losses for drift (Stage~2) and diffusion (Stage~3)
in Algorithm~\ref{alg:pipeline} are coordinate-invariant.

\section{The \texorpdfstring{$\rho$}{rho}-Metric and Geometric Penalties}\label{sec:rho-metric}

In this section we introduce the function-space metric
that underlies our geometric losses.
Section~\ref{ssec:rho-metric-def} motivates a well-conditioned
Jacobian assumption and defines the $\rho$-metric.
Section~\ref{ssec:efficient-computation} specifies the empirical
$\rho$-loss, details the data requirements, and shows that the
full pipeline is efficiently computable from dynamics observations.

\subsection{The \texorpdfstring{$\rho$}{rho}-metric on
  \texorpdfstring{$E^d(s)$}{Ed(s)}}\label{ssec:rho-metric-def}

Our metric will penalize the tangent-projector mismatch
$\norm{P_\phi-P_\psi}_F$.
A natural question is whether this quantity is controlled by the
Jacobian error $\norm{D\phi-D\psi}_F$.
In general the answer is \emph{no}: when the smallest singular value of
$D\phi$ degenerates, an arbitrarily small Jacobian perturbation can
rotate the tangent space by a large angle.
This motivates restricting attention to charts whose Jacobians have a
uniform lower singular-value bound.

Let $R^d(s):=\{X\in\R^{D\times d}:\sigmin(X)\ge s\}$
denote the set of full-rank matrices with controlled smallest singular value.
We first record two identities about the orthogonal
projection map $f(X) = X(X^\top X)^{-1}X^\top$ that
will be used throughout.

\begin{lem}[Frobenius distance of orthogonal projections]\label{lem:OP-frob-distance}
  Let $P_1$ and $P_2$ be $D\times D$, rank-$d$
  orthogonal projections.
  Write $P_i = H_i H_i^\top$ with $H_i^\top H_i = I_d$,
  and let $N_i$ be $D\times(D-d)$ with orthonormal
  columns spanning $(\rng P_i)^\perp$.
  Then
  \[
    \tfrac{1}{2}\norm{P_1-P_2}_F^2
    = d - \norm{H_1^\top H_2}_F^2
    = \norm{N_1^\top H_2}_F^2
    = \norm{N_2^\top H_1}_F^2.
  \]
\end{lem}

\begin{proof}
  Expanding gives
  $\norm{P_1-P_2}_F^2
  = 2(d - \frobinner{P_1}{P_2})$
  since $\norm{P_i}_F^2=d$.
  The cyclic property of the trace yields
  $\frobinner{P_1}{P_2}=\norm{H_1^\top H_2}_F^2$.
  For the normal form, use
  $I=H_iH_i^\top+N_iN_i^\top$ to decompose $H_2$ and
  note $\norm{H_2}_F^2=d$.
\end{proof}

The lower singular-value bound is necessary: without
it, the projection map is not Lipschitz.
Setting $\phi(t)=\epsilon\,t\,\bm{e}_1$ and
$\psi(t)=\epsilon\,t\,\bm{e}_2$ gives
$\norm{D\phi-D\psi}_F=\sqrt{2}\,\epsilon$ but
$\norm{P_\phi-P_\psi}_F=\sqrt{2}$, so the ratio
blows up as $\epsilon\to 0$.
On $R^d(s)$, however, the projection map is Lipschitz:

\begin{lem}[Lipschitz property of the projection map]\label{lem:OP-lipschitz}
  Let $f(X) = X(X^\top X)^{-1}X^\top$ be the orthogonal projection onto
  $\rng X$.
  Then for all $X,Y\in R^d(s)$,
  \[
    \norm{f(X)-f(Y)}_F
    \le \frac{\sqrt{2}}{s}\,\norm{X-Y}_F.
  \]
\end{lem}

\begin{proof}
  Write $g_Y=Y^\top Y$.
  Since $f(X)X=X$, we have $(I-f(X))X=0$, so by
  Lemma~\ref{lem:OP-frob-distance} applied with
  $H_2=Y\,g_Y^{-1/2}$,
  $\tfrac{1}{2}\norm{f(X)-f(Y)}_F^2
  = \norm{(I-f(X))\,Y\,g_Y^{-1/2}}_F^2
  = \norm{(I-f(X))\,(Y-X)\,g_Y^{-1/2}}_F^2$.
  Since $I-f(X)$ is a contraction and
  $\norm{g_Y^{-1/2}}_2=\sigmin(Y)^{-1}\le s^{-1}$,
  the right-hand side is at most
  $s^{-2}\norm{X-Y}_F^2$.
\end{proof}

With this Lipschitz bound in hand, we define the chart class and the
metric.
Let
$E^d(s):=\{\phi\in H^1(\Om;\R^D):D\phi(z)\in R^d(s)
\text{ a.e.\ on } \Om\}$
denote the set of well-conditioned charts, i.e.\ those whose Jacobians have
$\sigmin\ge s$ at almost every point.
The requirement is mild: any $C^1$
immersion~\cite{lee2012smooth}
$\phi_\star:\bar\Om\to\RD$ satisfies
$s_0:=\min_{\bar\Om}\sigmin(D\phi_\star)>0$ by the extreme value
theorem, so $\phi_\star\in E^d(s_0)$ automatically.
Every Monge patch
$\phi(z)=(z^\top,f(z))^\top$ satisfies $\sigmin(D\phi)\ge 1$
because $D\phi=[I_d;\,Df]$.
A smaller~$s_0$ enlarges the Lipschitz constant in
Lemma~\ref{lem:OP-lipschitz} and hence the generalization bounds.
In practice the assumption is not explicitly enforced during
training; we verify it post-hoc in
Section~\ref{sec:experiments}.
To use $E^d(s)$ as a hypothesis class for ERM, we first verify
that it is closed under $H^1$ limits.

\begin{lem}\label{lem:Ed-closed}
  $E^d(s)$ is closed in~$H^1$.
\end{lem}

\begin{proof}
  Let $\phi_k\in E^d(s)$ with $\phi_k\to\phi$ in $H^1$.
  By definition of the $H^1$ norm,
  $\norm{D\phi_k-D\phi}_{L^2}\le\norm{\phi_k-\phi}_{H^1}\to 0$,
  so $D\phi_k\to D\phi$ in $L^2(\Om;\R^{D\times d})$.
  By the $L^p$ subsequence theorem, there exists a subsequence
  $D\phi_{k_j}\to D\phi$ pointwise almost everywhere.
  Weyl's singular-value perturbation inequality gives
  $|\sigmin(A)-\sigmin(B)|\le\norm{A-B}_F$ for any matrices
  $A,B$ of the same size, so $\sigmin$ is $1$-Lipschitz on
  $\R^{D\times d}$.
  Therefore
  $\sigmin(D\phi_{k_j}(z))\to\sigmin(D\phi(z))$
  for almost every $z\in\Om$.
  Since $\sigmin(D\phi_{k_j}(z))\ge s$ for all~$j$, the limit
  satisfies $\sigmin(D\phi(z))\ge s$ almost everywhere, hence
  $\phi\in E^d(s)$.
\end{proof}

With the chart class in place, we combine the pointwise $L^2$
reconstruction error and the tangent-projector discrepancy from
Lemma~\ref{lem:OP-lipschitz} into a single metric.

\begin{lem}[The $\rho$-metric]\label{lem:rho-metric}
  For $\phi,\psi\in E^d(s)$, define
  \[
    \rho(\phi,\psi)^2
    := \norm{\phi-\psi}_{L^2}^2
     + \tfrac{1}{2}\,\norm{P_\phi - P_\psi}_{F,L^2}^2,
  \]
  where $P_\phi = f(D\phi)$ is the orthogonal projection onto
  $\rng(D\phi)$ as in Lemma~\ref{lem:OP-lipschitz}.
  Then the following hold:
  \textup{\textbf{(i)}}~$\rho$ is a metric on $E^d(s)$.\quad
  \textup{\textbf{(ii)}}~$\norm{\phi-\psi}_{L^2} \le \rho(\phi,\psi) \le
    C_s\,\norm{\phi-\psi}_{H^1}$
    with $C_s = \max(1,\,s^{-1})$.\quad
  \textup{\textbf{(iii)}}~$\rho$ is not equivalent to the $H^1$ norm.
\end{lem}

\begin{proof}
  \textbf{(i)}
  Non-negativity and symmetry are immediate.
  If $\rho(\phi,\psi)=0$ then both terms vanish; in particular
  $\norm{\phi-\psi}_{L^2}=0$, so $\phi=\psi$ a.e.\ on $\Om$ and hence
  $D\phi = D\psi$ a.e., giving $\phi=\psi$ in $H^1\cap E^d(s)$.
  For the triangle inequality, let $\phi,\psi,\chi\in E^d(s)$ and
  write $a_1=\norm{\phi-\psi}_{L^2}$,
  $b_1=\tfrac{1}{\sqrt{2}}\norm{P_\phi-P_\psi}_{F,L^2}$,
  and define $(a_2,b_2)$, $(a_3,b_3)$ analogously for
  $(\psi,\chi)$ and $(\phi,\chi)$.
  The $L^2$ triangle inequality gives $a_3\le a_1+a_2$ and
  $b_3\le b_1+b_2$, so $$\rho(\phi,\chi)
    = \sqrt{a_3^2+b_3^2}
    \le \sqrt{(a_1{+}a_2)^2+(b_1{+}b_2)^2}
    \le \rho(\phi,\psi)+\rho(\psi,\chi),$$
  where the last inequality is the triangle inequality for the
  Euclidean norm on~$\R^2$.

\smallskip
  \noindent\textbf{(ii)}
  The lower bound $\norm{\cdot}_{L^2}\le\rho$ is immediate from the
  definition.
  For the upper bound, Lemma~\ref{lem:OP-lipschitz} gives
  $\norm{P_\phi(x)-P_\psi(x)}_F \le \frac{\sqrt{2}}{s}\,
  \norm{D\phi(x)-D\psi(x)}_F$
  pointwise a.e.
  Squaring, integrating over $\Om$, and combining with the $L^2$ term
  yields
   $$ \rho(\phi,\psi)^2
    \le \norm{\phi-\psi}_{L^2}^2
     + \frac{1}{s^2}\,\norm{D\phi-D\psi}_{L^2}^2
    \le C_s^2\,\norm{\phi-\psi}_{H^1}^2.$$

\smallskip
	  \noindent\textbf{(iii)}
	Let $z=(z_1,\dots,z_d)\in \Om$ and define
	  $\phi(z)=\sum_{j=1}^d z_j\,\bm e_j $ and
	 $\phi_k(z)=\phi(z)-\frac{a}{k}\cos(k z_1)\,\bm e_1$
	  for $0<a<1-s$.
	  Then
	  $\partial_1\phi_k(z)=(1+a\sin(k z_1))\,\bm e_1$ and
	  $\partial_j\phi_k(z)=\bm e_j$ for $j=2,\dots,d$, so
	  $\sigmin(D\phi_k)\ge 1-a>s$ and $\norm{D\phi_k}_F$ is uniformly
	  bounded.  Hence $\phi_k\in E^d(s)$.
	  Both $P_{\phi_k}$ and $P_\phi$ equal the constant orthogonal
	  projection onto $\mathrm{span}\{\bm e_1,\dots,\bm e_d\}$, so
	  $\norm{P_{\phi_k}-P_\phi}_{F,L^2}=0$.
	  Moreover $\norm{\phi_k-\phi}_{L^2}\to 0$ as $k\to\infty$
	  (the oscillation has amplitude $a/k$).
	  Hence $\rho(\phi_k,\phi)\to 0$.
	  However,
	  $D\phi_k - D\phi$ has first column $a\sin(k z_1)\,\bm e_1$ and all
	  other columns zero, so
	   $$ \norm{D\phi_k-D\phi}_{L^2(\Om)}^2
	    = a^2\int_\Om \sin^2(k z_1)\,dz
	    = \frac{a^2}{2}\abs{\Om}
	      -\frac{a^2}{2}\int_\Om \cos(2k z_1)\,dz
	    \rightarrow \frac{a^2}{2}\abs{\Om}>0,$$
	  by the Riemann--Lebesgue lemma.
	  Thus $\phi_k\not\to\phi$ in $H^1$, so $\rho$ is strictly weaker
	  than~$H^1$.
	\end{proof}

In summary, $\rho$ sits strictly between $L^2$ and $H^1$ on
$E^d(s)$: it controls tangent-space alignment without requiring
full derivative matching.
This is precisely the gap exploited by the tangent-bundle penalty.

\subsection{Efficient computation from dynamics
  data}\label{ssec:efficient-computation}

The pointwise $\rho$-loss at a latent coordinate~$z$ is
$$\ell_\rho(z;\theta)
:= \norm{\phi_\theta(z) - \phi_\star(z)}^2
 + \tfrac{1}{2}\,\norm{P_{\phi_\theta}(z) - P_{\phi_\star}(z)}_F^2,$$
where $\phi_\star$ is the target chart and $\phi_\theta$ is the
decoder parametrized by~$\theta$.
For a probability measure $\nu$ on $\Om$, the
\emph{population risk} is
$R_{\nu,\rho}(\theta) := \int_\Om \ell_\rho(z;\theta)\,d\nu(z)$,
and the \emph{empirical risk minimizer} is
$\hat\theta_{S,\rho}
\in \arg\min_{\theta\in\Theta}
\tfrac{1}{m}\sum_{i=1}^{m}\ell_\rho(z_i;\theta)$
with $z_i\stackrel{\text{i.i.d.}}{\sim}\nu$.
Since $\ell_\rho$ is coordinate-invariant
(Lemma~\ref{lem:coord-invariance}), $R_{\nu,\rho}$ is
reparametrization-invariant whenever $\nu$ is the pullback of a
fixed measure on~$M$, in particular for the normalized Riemannian
volume $d\nu_\star \propto \sqrt{\det g_\star}\,dz$.
For the generalization analysis in
Section~\ref{sec:generalization}, we take $\nu$ to be
Lebesgue on $[0,1]^d$; the two are equivalent up to bounded
constants since $g_\star$ has bounded singular values on
compact~$\Om$.

The target chart~$\phi_\star$ is unknown in practice.
The training data $\{x_i, b(x_i), \Lambda(x_i)\}_{i=1}^m$ from
the ATLAS exploration phase provide two ingredients for
realizing the $\rho$-loss; the drift $b(x_i)$ enters in
Stage~2 (below).
Writing $z_i = \pi_\theta(x_i)$ and
$\hat P_i = U_d U_d^\top$ for the rank-$d$ spectral projector
of~$\Lambda(x_i)$ (Section~\ref{ssec:tangent-from-cov}):
\begin{itemize}
\item \emph{Reconstruction}\;
$\cL_R^{(i)} = \norm{\phi_\theta(z_i) - x_i}^2$:
a surrogate for the $\rho$-loss term
$\norm{\phi_\theta(z) - \phi_\star(z)}^2$, since
$x_i = \phi_\star(z_i^\star)$ for the true latent
$z_i^\star = \pi_\star(x_i)$.
The two coincide when $\pi_\theta = \pi_\star$.
\item \emph{Tangent-bundle}\;
$\cL_T^{(i)} = \tfrac{1}{2}\norm{P_{\phi_\theta}(z_i) - \hat P_i}_F^2$:
replaces the unknown $P_{\phi_\star}$ with the data-derived~$\hat P_i$.
\end{itemize}
We augment the $\rho$-loss with an
\emph{inverse-consistency penalty}~\cite{greer2021icon},
$\cL_F^{(i)} = \norm{D\pi_\theta(x_i)\,D\phi_\theta(z_i) - I_d}_F^2$,
which enforces the first-order inverse identity
$D\pi\,D\phi = I_d$, ensuring that the encoder and decoder
are consistent as approximate chart maps.
The condition $D\pi\,D\phi = I_d$ is coordinate-invariant
(it holds in all parameterizations or none),
though the Frobenius penalty itself is not;
$\cL_F$ is a practical regularizer, not part of the
$\rho$-metric or the generalization theory.
The full training objective for Stage~1 (chart learning) is
\begin{equation}\label{eq:training-objective}
  \cL(\theta)
  = \frac{1}{m}\sum_{i=1}^{m}
    \bigl[\,
      \cL_R^{(i)}
      + \lambda_T\,\cL_T^{(i)}
      + \lambda_F\,\cL_F^{(i)}
    \,\bigr].
\end{equation}
Stages~2 and~3 freeze the chart and fit latent SDE coefficients
using coordinate-invariant losses
(Lemma~\ref{lem:coord-invariance}\,(ii)--(iii)).
Stage~2 regresses a drift network $\hat\mu_\omega(z)$
onto the encoder-pullback target~\eqref{eq:enc-pull-target}
under $\norm{\cdot}_g^2$;
Stage~3 regresses a diffusion network
$\hat\sigma_\psi(z)$ onto the pulled-back covariance
$\Sigma(z_i) := D\pi_\theta(x_i)\,\Lambda(x_i)\,D\pi_\theta(x_i)^\top$
under $\tr\bigl((g\,\cdot\,)^2\bigr)$.
The geometric losses are coordinate-invariant:
$\cL_R$ is computed in ambient~$\RD$;
$\cL_T$ uses the ambient projector
(Lemma~\ref{lem:coord-invariance}\,(i));
Stage~2 and~3 use metric-weighted norms
(Lemma~\ref{lem:coord-invariance}\,(ii)--(iii)).
The law of the learned ambient process is therefore
independent of the latent parameterization.
The auxiliary penalty $\cL_F$ enforces a
coordinate-invariant condition ($D\pi\,D\phi=I_d$) but its
Frobenius value is chart-dependent.
The full pipeline is summarized in Algorithm~\ref{alg:pipeline}.

\begin{algorithm}[t]
\caption{Three-stage pipeline: geometric chart learning and latent SDE estimation}
\label{alg:pipeline}
\begin{algorithmic}[1]
\REQUIRE Ambient observations $\{x_i\}_{i=1}^m\subset\RD$ with
  drift $b(x_i)$ and covariance $\Lambda(x_i)$
  \COMMENT{e.g.\ from short-burst ensembles~\cite{Miles17,yym}}
\REQUIRE Penalty weights $\lambda_T,\lambda_F$;
  epochs $E$
\ENSURE Chart $(\pi_\theta,\phi_\theta)$; latent SDE coefficients
  $(\hat\mu_\omega,\hat\sigma_\psi)$
\STATE \textbf{Stage~1: Chart learning} \COMMENT{Geometric autoencoder}
\STATE Train $(\pi_\theta,\phi_\theta)$ in two phases
  (warmup $E_1$, fine-tune $E_2$ epochs),
  both minimizing
  $\;\cL=\cL_R+\lambda_T\cL_T+\lambda_F\cL_F$
  via~\eqref{eq:training-objective}
\STATE \textbf{Stage~2: Encoder-pullback drift fitting} \COMMENT{Frozen chart}
\STATE Encode $z_i\gets\pi_\theta(x_i)$; compute encoder-pullback target
  $\mu(z_i) = D\pi_\theta(x_i)\,b(x_i) + \tfrac{1}{2}\frobinner{\Lambda(x_i)}{\nabla^2\pi_\theta(x_i)}$
  via~\eqref{eq:enc-pull-target}
\STATE Train $\hat\mu_\omega$ to minimize
  $\sum_i\norm{\hat\mu_\omega(z_i)-\mu(z_i)}_{g(z_i)}^2$
\STATE \textbf{Stage~3: Latent diffusion fitting} \COMMENT{Frozen chart}
\STATE Compute target $\Sigma(z_i) = D\pi_\theta(x_i)\,\Lambda(x_i)\,D\pi_\theta(x_i)^\top$
\STATE Train $\hat\sigma_\psi$ to minimize
  $\sum_i\tr\bigl((g(z_i)\,[\hat\sigma_\psi(z_i)\hat\sigma_\psi(z_i)^\top - \Sigma(z_i)])^2\bigr)$
\end{algorithmic}
\end{algorithm}

Na\"ively, the tangent-bundle penalty forms $D\times D$ projection
matrices ($O(D^2 d)$ time, $O(D^2)$ memory).
The following proposition shows that all stages can be evaluated
with $O(Dd^2)$ flops and $O(Dd)$ memory.

\begin{prop}[Computational cost]\label{prop:efficient-tangent}%
\label{prop:hessian-free}
  Every loss and target in the three-stage pipeline
 can be evaluated in $O(Dd^2)$
  flops per sample, with $O(Dd)$ for the drift loss, without
  forming any $D\times D$ matrix.
  Specifically, with $g=D\phi^\top D\phi\in\R^{d\times d}$:
  \textbf{(R)} reconstruction is $O(D)$;
  \textbf{(T)} tangent-bundle is $O(Dd^2)$;
  \textbf{(F)} inverse-consistency is $O(Dd^2)$;
  \textbf{(Stage~2)} the encoder-pullback target is $O(Dd^2)$
  via $O(d^2)$ Hessian--vector products at $O(D)$ each;
  the drift loss itself is $O(Dd)$ given
  the precomputed Jacobian $D\phi\in\R^{D\times d}$;
  \textbf{(Stage~3)} pulled-back covariance target and
  diffusion loss are each $O(Dd^2)$.
\end{prop}

\begin{proof}
  Write $U_d\in\R^{D\times d}$ for the top-$d$ eigenvectors
  of~$\Lambda$, $P=U_d\,U_d^\top$, $\hat P=D\phi\,g^{-1}D\phi^\top$,
  and $C=D\phi^\top U_d\in\R^{d\times d}$.

\smallskip
  \textbf{(R)}
  $\cL_R = \norm{\phi_\theta(z)-x}^2$ is a single $D$-vector
  difference and dot product: $O(D)$.

\smallskip
  \textbf{(T)}
  Since $\hat P$ and $P$ are both orthogonal projections of
  rank~$d$, idempotency gives $\hat P^2=\hat P$, $P^2=P$, and
  $\tr(\hat P)=\tr(P)=d$.
  Expanding:
  $\norm{\hat P-P}_F^2
  = \tr(\hat P^2) - 2\,\tr(\hat P\,P) + \tr(P^2)
  = 2d - 2\,\tr(\hat P\,P)$.
  Substituting $\hat P=D\phi\,g^{-1}D\phi^\top$ and
  $P=U_d\,U_d^\top$ and applying the cyclic property of the trace:
  $\tr(\hat P\,P)
  = \tr\bigl(D\phi\,g^{-1}D\phi^\top U_d\,U_d^\top\bigr)
  = \tr\bigl(g^{-1}\,(D\phi^\top U_d)(D\phi^\top U_d)^\top\bigr)
  = \tr(g^{-1}CC^\top)$,
  so
  \begin{equation}\label{eq:tangent-trace}
    \tfrac{1}{2}\norm{\hat P-P}_F^2
    = d - \tr\bigl(g^{-1}C\,C^\top\bigr).
  \end{equation}
  Forming $C=D\phi^\top U_d$ costs $O(Dd^2)$;
  $g^{-1}CC^\top$ is $d\times d$,
  so the trace is $O(d^3)\subset O(Dd^2)$.
  When the inverse-consistency penalty drives
  $D\pi\,D\phi\approx I_d$, we have
  $D\pi\approx g^{-1}D\phi^\top$ and the loss
  simplifies to $d - \tr(D\pi\,U_d\,U_d^\top D\phi)$,
  which is the form used in Algorithm~\ref{alg:pipeline}.

\smallskip
  \textbf{(F)}
  $\cL_F = \norm{D\pi\,D\phi - I_d}_F^2$:
  $D\pi\in\R^{d\times D}$ times $D\phi\in\R^{D\times d}$
  is a $d\times d$ product costing $O(Dd^2)$.

\smallskip
  \textbf{(Stage~2)}
  \emph{Target.}
  The encoder-pullback drift~\eqref{eq:enc-pull-target} contains
  the Itô correction $\frobinner{\Lambda}{\nabla^2\pi^j}$.
  Since $\Lambda$ has rank~$d$ on the manifold,
  its spectral decomposition
  $\Lambda = \sum_{m=1}^d \lambda_m u_m u_m^\top$ gives
  $\frobinner{\Lambda}{\nabla^2\pi^j}
  = \sum_{m=1}^d \lambda_m\,u_m^\top(\nabla^2\pi^j)\,u_m$,
  where each $u_m^\top(\nabla^2\pi^j)\,u_m$ is a single
  forward-over-reverse Hessian--vector product (HVP), avoiding
  the full $d\times D\times D$ encoder Hessian.
  Over $j=1,\dots,d$ encoder outputs this is $d^2$ HVPs at $O(D)$
  each, totalling $O(Dd^2)$.
  \emph{Loss.}
  $\norm{r}_g^2 = r^\top g\,r = \norm{D\phi\,r}^2$
  rewrites the metric-weighted norm as an ambient Euclidean norm.
  Computing $D\phi\,r$ is a matrix--vector product in $O(Dd)$;
  the squared norm costs $O(D)$.

\smallskip
  \textbf{(Stage~3)}
  The pulled-back covariance target
  $\Sigma(z_i) = D\pi\,\Lambda\,D\pi^\top$ is naively $O(dD^2)$
  because $\Lambda\in\R^{D\times D}$.
  Since $\rnk(\Lambda)=d$, write
  $\Lambda = U_d\,\diag(\lambda_1,\dots,\lambda_d)\,U_d^\top$
  and set $B = D\pi\,U_d\in\R^{d\times d}$ at cost $O(Dd^2)$;
  then $\Sigma = B\,\diag(\lambda)\,B^\top$ in $O(d^3)$.
  For the loss, $g = D\phi^\top D\phi$ is precomputed in $O(Dd^2)$;
  the product $g\,\Delta\Sigma$ is a $d\times d$ matrix
  multiplication costing $O(d^3)$.
  Squaring $(g\,\Delta\Sigma)^2$ and taking the trace are
  likewise $O(d^3) \subset O(Dd^2)$.
\end{proof}

\section{Generalization Theory}\label{sec:generalization}

The previous section constructed a pipeline
whose geometric losses ($\cL_R$, $\cL_T$, Stage~2, Stage~3)
are coordinate-invariant and depend only on point-cloud
positions, local covariance eigenvectors, ambient drift,
and the learned chart; no ground-truth Jacobians or
curvature labels are required.
Yang et al.~\cite{yang2025deeperwiderperspectiveoptimal} establish
optimal generalization rates for deep networks trained in the full
Sobolev $H^k$ norm, where the training loss requires labels for all
derivatives up to order~$k$.
Our $\rho$-loss is strictly less supervised than $H^1$
training, which would require full Jacobian labels
$D\phi_\star(x_i)$.
A natural question is whether this weaker supervision comes at a
statistical cost.

Perhaps surprisingly, the answer is that no rate is lost.
We show that $\rho$-ERM achieves the \emph{same} rate (up to
logarithmic factors) as full $H^1$ training.
The key new ingredient is a bridge lemma
(Lemma~\ref{lem:rho-approx-from-W1infty}) that converts a
$W^{1,\infty}$ approximation guarantee, available from
\cite{yang2025deeperwiderperspectiveoptimal}, into a
$\rho$-approximation guarantee.
Once this bridge is in place, the estimation error is controlled
by a covering-number argument over the first-order feature class,
which reduces to the entropy bounds already established
in~\cite{yang2025deeperwiderperspectiveoptimal}.
This improves on the chart autoencoder
rate of Liu et
al.~\cite{liu2023deepnonparametricestimationintrinsic}, which
achieves $m^{-2/(d+2)}\log^4 m$ from reconstruction loss alone:
geometric supervision via the tangent projector replaces the
fixed smoothness $n=2$ in their bound with the true regularity~$n$
of the target chart.
Section~\ref{ssec:bias-analysis} then analyzes the systematic bias
that arises when the decoder-side drift formula is used in place of
the encoder-pullback target.
We begin with three standing assumptions.\label{ssec:assumptions}

\begin{assumption}[Target chart regularity]\label{ass:target-chart}
Let $\Om\subset\Rd$ be a bounded Lipschitz domain and let
$\phi_\star:\Om\to\RD$ be the ground-truth chart.  Assume
$\phi_\star\in E^d(s_0)\cap W^{n,\infty}(\Om;\RD)$ for some
$s_0>0$ and $n>1$, and after rescaling,
$ \norm{\phi_\star}_{W^{n,\infty}(\Om)}\le 1$.

The well-conditioned chart class $E^d(s_0)$ from
Section~\ref{ssec:rho-metric-def} ensures that
$\sigmin(D\phi_\star)\ge s_0$ a.e.\ on~$\Om$, so the $\rho$-metric
is well-defined on a neighborhood of~$\phi_\star$.
\end{assumption}

\begin{assumption}[Hypothesis class]\label{ass:rho-class}
Fix constants $B\ge 1$ and $s\in(0,s_0]$.  For each architecture
$(N,L)$ (width $N$, depth $L$), let
$\cH(N,L)=\{\phi_\theta:\Om\to\RD:\theta\in\Theta_{N,L}\}$ be a
DeNN hypothesis set.  The \emph{constrained} class used for
ERM is
\[
  \cH_{B,s}(N,L)
  :=\Big\{\phi_\theta\in\cH(N,L):\
    \norm{\phi_\theta}_{W^{1,\infty}(\Om)}\le B
    \ \text{and}\ \phi_\theta\in E^d(s)\Big\}.
\]
Since $\cH_{B,s}(N,L)\subset E^d(s)$, the $\rho$-metric and the
projector $P_{\phi_\theta}$ are well-defined for every member.
\end{assumption}

\begin{assumption}[$W^{1,\infty}$ approximation]\label{ass:W1infty-approx}
Assume the DeNN architecture is chosen as in
\cite{yang2025deeperwiderperspectiveoptimal}
(width $N=O(\log L)$, depth~$L$, parameter count
$W=O(L(\log L)^3)$).
By the $W^{1,\infty}$ approximation result
(Proposition~A.2 and Theorem~3.2
of~\cite{yang2025deeperwiderperspectiveoptimal}),
there exists a
constant $C_{\mathrm{app}}>0$ depending only on $(d,D,n)$ and fixed
architectural constants such that, for each $(N,L)$, one can find
$\theta^\sharp\in\Theta_{N,L}$ with
$\phi_{\theta^\sharp}\in\cH_{B,s}(N,L)$ and
$\norm{\phi_{\theta^\sharp}-\phi_\star}_{W^{1,\infty}(\Om)}
\le
C_{\mathrm{app}}
\bigl(\frac{W}{(\log W)^2}\bigr)^{-{2(n-1)}/{d}}$.
Once the right-hand side is at most~$s_0/2$, Weyl's perturbation
bound~\cite{horn2013matrix} gives
$\phi_{\theta^\sharp}\in E^d(s_0/2)$, so membership in the
constrained class is automatic.\label{rem:replace-Ed}
\end{assumption}

The singular-value bound defining $E^d(s)$ depends on the choice of
coordinates on~$\Om$; the latent domain with its Euclidean structure
is fixed as part of the statistical setup.
As noted in Section~\ref{sec:rho-metric}, the $\rho$-loss
and Stage~2/3 losses are
coordinate-invariant
(Lemma~\ref{lem:coord-invariance}), and under
bounded-distortion reparametrisations the hypothesis class is
stable: $\phi\in E^d(s)\Rightarrow\phi\circ\psi\in E^d(as)$ when
$\sigmin(D\psi)\ge a$.

\subsection{Main generalization results}%
\label{ssec:rho-from-W1infty}\label{ssec:main-gen-thm}

The first ingredient is a bridge lemma showing that
$W^{1,\infty}$-closeness to a well-conditioned target implies
$\rho$-closeness.

\begin{lem}\label{lem:rho-approx-from-W1infty}
Let $\Om\subset\Rd$ be measurable with $0<\abs{\Om}<\infty$ and let
$\phi_\star:\Om\to\RD$ be a $C^1$ chart such that for all $x\in \Om$,
the Jacobian $J_\star(x):=D\phi_\star(x)\in\R^{D\times d}$ has full
rank~$d$ and
$s_0 \;\le\; \sigmin(J_\star(x))
  \;\le\; \sigmax(J_\star(x)) \;\le\; r_0$
  for some constants $0<s_0\le r_0<\infty$. 
For any $C^1$ map $\phi:\Om\to\RD$ with $D\phi(x)$ full rank, define
$P_\phi(x)
  \;:=\;
D\phi(x)\big(D\phi(x)^\top D\phi(x)\big)^{-1}D\phi(x)^\top$ and
  $\rho(\phi,\phi_\star)^2
  \;:=\;
  \norm{\phi-\phi_\star}_{L^2(\Om)}^2
  +\tfrac12\norm{P_\phi-P_{\phi_\star}}_{F,L^2(\Om)}^2$.
Assume there exists a neural-network realizable
$\phi=\phi_{N,L}$ satisfying
\begin{equation}\label{eq:W1infty-approx}
  \norm{\phi-\phi_\star}_{W^{1,\infty}(\Om)} \;\le\; \delta,
\end{equation}
and suppose $\delta \le s_0/2$.  Then
\textup{\textbf{(i)}}~For all $x\in \Om$,
$\tfrac{s_0}{2}
\le \sigmin(D\phi(x))
\le \sigmax(D\phi(x))
\le r_0+\delta$,
so $D\phi(x)$ is full rank and $P_\phi$ is well-defined.
\textup{\textbf{(ii)}}~There is a constant
$C_{\rho}(s_0,\Om):=\abs{\Om}\bigl(1+\frac{4}{s_0^2}\bigr)$
such that
\begin{equation}\label{eq:rho-from-W1infty}
  \rho(\phi,\phi_\star)^2
  \;\le\;
  C_{\rho}(s_0,\Om)\,
  \norm{\phi-\phi_\star}_{W^{1,\infty}(\Om)}^2
  \;\le\;
  C_{\rho}(s_0,\Om)\,\delta^2.
\end{equation}

Consequently, if an approximation theorem yields
$\delta \lesssim N^{-2(n-1)/d}L^{-2(n-1)/d}$ for
$\phi_\star\in W^{n,\infty}$, then
  $\rho(\phi,\phi_\star)^2
  \;\lesssim\;
  N^{-4(n-1)/d}L^{-4(n-1)/d}$.
In the DeNN regime where $N=O(\log L)$ and the parameter count
satisfies $W=O(N^2L\log L)$, the above can be rewritten (up to
logarithmic factors) as
$\rho(\phi,\phi_\star)^2
  \;\lesssim\;
  \Big(\frac{W}{(\log W)^2}\Big)^{-\,\frac{4(n-1)}{d}}$.
\end{lem}

\begin{proof}
Write $J(x):=D\phi(x)$ and $E(x):=J(x)-J_\star(x)$.

\noindent \textbf{(i))}
For each fixed $x\in \Om$, Weyl's perturbation bound for singular
values gives
$\abs{\sigma_i(J(x))-\sigma_i(J_\star(x))}
  \;\le\; \norm{E(x)}_2$,
for each singular value index~$i$.  Since
$\norm{E(x)}_2\le \norm{E(x)}_F
 \le \norm{E}_{F,\infty}\le \delta$
by~\eqref{eq:W1infty-approx}, we obtain
$\sigmin(J(x))
  \;\ge\; \sigmin(J_\star(x))-\delta
  \;\ge\; s_0-\delta
  \;\ge\; \frac{s_0}{2}$,
and similarly
$\sigmax(J(x))
  \;\le\; \sigmax(J_\star(x))+\delta
  \;\le\; r_0+\delta$. 
This proves~(i).

\noindent \textbf{(ii))}
Let $f(X):=X(X^\top X)^{-1}X^\top$ be the projector map.  By the
Lipschitz bound for the projector on matrices with
$\sigmin(\cdot)\ge s_0/2$ (Lemma~\ref{lem:OP-lipschitz}), we have
 $ \norm{f(X)-f(Y)}_F
  \;\le\;
  \frac{\sqrt{2}}{s_0/2}\,\norm{X-Y}_F
  \;=\;
  \frac{2\sqrt{2}}{s_0}\,\norm{X-Y}_F$.
Applying this with $X=J(x)$ and $Y=J_\star(x)$ yields, for all
$x\in \Om$,
  $\norm{P_\phi(x)-P_{\phi_\star}(x)}_F
  \;=\;\norm{f(J(x))-f(J_\star(x))}_F
  \;\le\;\frac{2\sqrt{2}}{s_0}\,\norm{J(x)-J_\star(x)}_F
  \;\le\;\frac{2\sqrt{2}}{s_0}\,\delta$.
Therefore
  $\norm{P_\phi-P_{\phi_\star}}_{F,L^2(\Om)}^2
  =\int_\Om \norm{P_\phi(x)-P_{\phi_\star}(x)}_F^2\,dx
  \;\le\;
  \abs{\Om}\cdot \frac{8}{s_0^2}\,\delta^2$.

Also by~\eqref{eq:W1infty-approx},
  $\norm{\phi-\phi_\star}_{L^2(\Om)}^2
  \;\le\;
  \abs{\Om}\cdot \norm{\phi-\phi_\star}_{L^\infty(\Om)}^2
  \;\le\;
  \abs{\Om}\,\delta^2$.
Combining the two bounds,
  $\rho(\phi,\phi_\star)^2
  =\norm{\phi-\phi_\star}_{L^2(\Om)}^2
   +\tfrac12\norm{P_\phi-P_{\phi_\star}}_{F,L^2(\Om)}^2
  \;\le\;
  \abs{\Om}\delta^2
+\tfrac12\abs{\Om}\frac{8}{s_0^2}\delta^2
  \;=\;
  \abs{\Om}\Big(1+\frac{4}{s_0^2}\Big)\delta^2$,
which is~\eqref{eq:rho-from-W1infty} and completes the proof.
\end{proof}

With this bridge in hand, we state the main result
for an \emph{oracle} $\rho$-ERM:
$\ell_\rho(z;\theta)$ evaluates $\phi_\theta$ against
$\phi_\star$ at latent points $z\sim\nu$ known to the
oracle, a standard device in covering-number oracle
inequalities; quasi-uniform ATLAS
designs~\cite{Gyrfi2002ADT} achieve comparable rates.
The implemented loss differs in two ways.
First, the learned encoder replaces the true latent
coordinate.  On the well-conditioned class, both the
reconstruction and projector terms in $\ell_\rho$ are
Lipschitz in the evaluation point, so the practical
$\rho$-risk exceeds the oracle risk by a term
controlled by
$\E\norm{\pi_\theta(\phi_\star(z))-z}^2$;
when $D\pi\,D\phi\approx I_d$ (enforced by $\cL_F$),
this gap is small because the encoder is an approximate
local inverse.
Second, if $\hat P_i$ is formed from an empirical
covariance $\hat\Lambda_i$, the Davis--Kahan
$\sin\Theta$ theorem yields
$\norm{\hat P_i - P_{\phi_\star}(z_i)}_F
 \le C\norm{\hat\Lambda_i-\Lambda_i}_{\mathrm{op}}
     /(\lambda_d-\lambda_{d+1})$,
where $\lambda_d{-}\lambda_{d+1}$ is the eigengap
separating tangent from normal eigenvalues; thus the
projector error vanishes when the covariance estimator
is accurate and the eigengap stays bounded away from
zero.  

\begin{thm}[Generalization in $\rho$ has the same order as
$H^1$]\label{thm:rho-gen-same-order-rigorous}
Under Assumptions~\ref{ass:target-chart}--\ref{ass:W1infty-approx},
let $\hat\theta$ be the $\rho$-ERM over
$\cH_{B,s_0/2}(N,L)$ with $\{z_i\}_{i=1}^m$ drawn i.i.d.\ from a
probability measure $\nu$ on~$\Om$.
Assume $N,L$ are large enough that
$C_{\mathrm{app}}\,N^{-{2(n-1)}/{d}}\,L^{-{2(n-1)}/{d}}\le s_0/2$.
Then there exists a constant $C>0$ depending only on
$(d,D,n,B,s_0)$ such that
\begin{equation}\label{eq:rho-NL-bound-rigorous}
  \E\,R_{\nu,\rho}(\hat\theta)
  \ \le\
  C\Bigg[
    N^{-\frac{4(n-1)}{d}}\,L^{-\frac{4(n-1)}{d}}
    \;+\;
    \frac{N^2L^2\,\log_2 L\,\log_2 N}{m}\,\log m
  \Bigg],
\end{equation}
where the expectation is over $(z_i)_{i=1}^m\stackrel{\text{i.i.d.}}{\sim}\nu$.

In particular, in the DeNN regime $N=O(\log L)$ and parameter count
$W=O(N^2L\log L)$, the bound~\eqref{eq:rho-NL-bound-rigorous} can be
rewritten (up to logarithmic factors) as
\begin{equation}\label{eq:rho-WM-bound-rigorous}
  \E\,R_{\nu,\rho}(\hat\theta)
  \ \le\
  C\Bigg[
    \Big(\tfrac{W}{(\log W)^2}\Big)^{-\frac{4(n-1)}{d}}
    \;+\;
    \frac{W^2}{m}\,\log m
  \Bigg],
\end{equation}
and optimizing over $W$ yields the same sample-only rate (up to logs)
as $H^1$-Sobolev training:
  $\E\,R_{\nu,\rho}(\hat\theta)
  \ \lesssim\ m^{-\frac{2(n-1)}{2(n-1)+d}}$.
\end{thm}

The $\rho$-metric discards tangential stretching
information and is therefore strictly weaker than $H^1$
(Lemma~\ref{lem:rho-metric}\,(iii)).
That it achieves the same rate reflects two complementary
facts.
On the \emph{approximation} side, the
well-conditioning assumption $\sigmin(D\phi)\ge s$
makes the map $D\phi\mapsto P_\phi$ Lipschitz
(Lemma~\ref{lem:OP-lipschitz}), so any $W^{1,\infty}$
approximant at scale~$\delta$ automatically gives
$\rho$-error at scale~$\delta^2$ with no loss in
exponent.
On the \emph{estimation} side, the $\rho$-loss is a
Lipschitz function of the same first-order features
$(\phi,D\phi)$ that determine the $H^1$ covering number,
so the entropy bound is inherited unchanged.
In short, $\rho$ uses less supervision than
$H^1$ because the covariance data supplies the
tangent-space information that would otherwise require
Jacobian labels, yet on the well-conditioned class both
exponents remain first-order.

\begin{proof}
Let $\hat\theta$ be the $\rho$-ERM and write
  $R_{\nu,\rho}^\star
  :=\inf\{R_{\nu,\rho}(\theta):\phi_\theta\in\cH_{B,s_0/2}(N,L)\}$.
Let
  $\mathcal L_\rho:=\{\ell_\rho(\cdot;\theta):\phi_\theta\in\cH_{B,s_0/2}(N,L)\}$.
Since the squared loss is uniformly bounded,
$0\le \ell_\rho\le B_\rho^2$ by
eq.~\eqref{eq:loss-bounded} below, we may apply a fast-rate covering-number
oracle inequality for ERM with bounded square
loss~\cite[Theorem~11.3]{Gyrfi2002ADT,bartlett2005local}.
Concretely, there exist
constants $C,c_0>0$ such that
\begin{equation}\label{eq:oracle-erm}
  \E\,R_{\nu,\rho}(\hat\theta)
  \ \le\
  R_{\nu,\rho}^\star
  \;+\;
  C\,\frac{1}{m}\Big(\log \cN\big(c_0/m,\mathcal L_\rho,m\big)+1\Big).
\end{equation}
Since $R_{\nu,\rho}^\star\le R_{\nu,\rho}(\bar\theta)$ for any fixed
comparator $\bar\theta$ in the class, it remains to (i)~choose
$\bar\theta$ to control the approximation error and (ii)~bound
$\log\cN(\cdot,\mathcal L_\rho,m)$.

By Assumption~\ref{ass:W1infty-approx}, there exists $\theta^\sharp$ such that
$\phi_{\theta^\sharp}\in\cH_{B,s_0/2}(N,L)$ and
$\norm{\phi_{\theta^\sharp}-\phi_\star}_{W^{1,\infty}(\Om)}
  \le
  C_{\mathrm{app}}\,N^{-\frac{2(n-1)}{d}}\,L^{-\frac{2(n-1)}{d}}$.
Denote this bound by $\delta$; by hypothesis $\delta\le s_0/2$.
The proof of Lemma~\ref{lem:rho-approx-from-W1infty} gives the
pointwise bound $\ell_\rho(z;\theta^\sharp)\le(1+4/s_0^2)\,\delta^2$;
integrating against the probability measure~$\nu$ yields
\[
  R_{\nu,\rho}(\theta^\sharp)
  \le
  (1+4/s_0^2)\,\delta^2
  \le
  C\,N^{-\frac{4(n-1)}{d}}\,L^{-\frac{4(n-1)}{d}}.
\]
Let $\bar\theta:=\theta^\sharp$.  Then
\begin{equation}\label{eq:approx-term}
  R_{\nu,\rho}(\bar\theta)
  \le C\,N^{-\frac{4(n-1)}{d}}\,L^{-\frac{4(n-1)}{d}}.
\end{equation}

First note $\ell_\rho(z;\theta)$ is uniformly bounded: since
$\norm{\phi_\theta}_{L^\infty}\le B$ and
$\norm{\phi_\star}_{L^\infty}\le 1$, we have
$\norm{\phi_\theta-\phi_\star}\le B{+}1$, and since
$P_{\phi_\theta},P_{\phi_\star}$ are orthogonal projections,
$\norm{P_{\phi_\theta}-P_{\phi_\star}}_F\le 2\sqrt{d}$, hence
\begin{equation}\label{eq:loss-bounded}
  0\le \ell_\rho(z;\theta)\le B_\rho^2
  \qquad\text{for all }z\in \Om,\ \theta,
\end{equation}
for a constant $B_\rho=B_\rho(d,B)$.

Next, by the Lipschitz bound for the projector map on the set
$\{J:\sigmin(J)\ge s_0/2\}$,
  $\norm{P_{\phi_\theta}(z)-P_{\phi_{\theta'}}(z)}_F
  \le \frac{C}{s_0}\,
    \norm{D\phi_\theta(z)-D\phi_{\theta'}(z)}_F
  \quad\text{for a.e. }z$.
Using the elementary inequality
$\big|\norm{a}^2-\norm{b}^2\big|
 \le (\norm{a}+\norm{b})\,\norm{a-b}$
and the uniform bounds in $W^{1,\infty}$, there exists
$C=C(d,D,B,s_0)$ such that for a.e.\ $z$ and all
$\theta,\theta'$,
\begin{equation}\label{eq:loss-Lipschitz-features}
  \abs{\ell_\rho(z;\theta)-\ell_\rho(z;\theta')}
  \le
  C\Big(\norm{\phi_\theta(z)-\phi_{\theta'}(z)}
       +\norm{D\phi_\theta(z)-D\phi_{\theta'}(z)}_F\Big).
\end{equation}

Define the first-order feature map class
  $\mathcal F
  :=\Big\{z\mapsto
    (\phi_\theta(z),
     \partial_1\phi_\theta(z),\dots,
     \partial_d\phi_\theta(z)):
    \phi_\theta\in\cH_{B,s_0/2}(N,L)\Big\}$.
Then~\eqref{eq:loss-Lipschitz-features} implies the uniform covering
numbers satisfy
$\cN\big(\varepsilon,\mathcal L_\rho,m\big) \le \cN\big(c\,\varepsilon,\mathcal F,m\big)$
for some constant $c=c(d,D,B,s_0)>0$, where
$\cN(\cdot,\cdot,m)$ denotes the uniform covering number on
$m$ sample points.

Bounding $\cN(\varepsilon,\mathcal F,m)$ is exactly the
``first-derivative'' ($k{=}1$) instance of the Sobolev-loss
generalization machinery
in~\cite[Theorem~3.2]{yang2025deeperwiderperspectiveoptimal}.
That result bounds the covering number of the full DeNN
derivative class; since
$\cH_{B,s_0/2}(N,L)$ is a subclass, the same entropy
bound applies \emph{a fortiori}.
Moreover, \cite[Remark~3.1]{yang2025deeperwiderperspectiveoptimal}
explains that the same entropy bound holds for general i.i.d.\
sampling measures.
Specializing their result to the first-order feature class
$\mathcal F$ yields, for $\varepsilon$ of order $1/m$,
\begin{equation}\label{eq:entropy-final}
  \log \cN(\varepsilon,\mathcal L_\rho,m)
  \ \le\
  C\,N^2L^2\log_2 L\log_2 N\cdot \log m,
\end{equation}
where $C=C(d,D,B,s_0)$.
Inserting~\eqref{eq:entropy-final} into~\eqref{eq:oracle-erm} gives
\begin{equation}\label{eq:sample-term}
  \E\,R_{\nu,\rho}(\hat\theta)
  \ \le\
  R_{\nu,\rho}^\star
  \;+\;
  C\,\frac{N^2L^2\log_2 L\log_2 N}{m}\,\log m.
\end{equation}

Since $R_{\nu,\rho}^\star\le R_{\nu,\rho}(\bar\theta)$,
combining~\eqref{eq:approx-term} with~\eqref{eq:sample-term}
yields~\eqref{eq:rho-NL-bound-rigorous}.  In the DeNN regime
$N=O(\log L)$ and $W=O(N^2L\log L)$,
\eqref{eq:rho-WM-bound-rigorous} follows up to logarithmic factors.
Balancing the two terms in~\eqref{eq:rho-WM-bound-rigorous} gives the
sample-only rate $m^{-\frac{2(n-1)}{2(n-1)+d}}$ (up to logs),
matching the Sobolev $H^1$ training rate.
\end{proof}

\subsection{Bias analysis of decoder-side drift fitting}%
\label{ssec:bias-analysis}

Theorem~\ref{thm:rho-gen-same-order-rigorous} guarantees that the
oracle $\rho$-ERM chart converges at the optimal rate.
The next question is how to extract latent SDE coefficients from
this chart.
The latent drift admits two equivalent representations
as derived in Section~\ref{ssec:enc-pull}: the encoder-pullback
form~\eqref{eq:enc-pull} using $(D\pi, \nabla^2\pi)$, and the
decoder form of Lemma~\ref{lem:ambient-to-local} using the metric
pseudo-inverse $g^{-1}D\phi^\top$ and the decoder Hessian.
For an ideal autoencoder ($\pi=\phi^{-1}$) these coincide; for a
learned autoencoder they diverge, and the decoder-side target
carries a systematic bias.

\begin{prop}[Decoder-side bias decomposition]\label{prop:dec-bias}
  Let $(\phi_\theta,\pi_\theta)$ be a $C^2$ encoder--decoder pair
  with $D\phi_\theta$ of full column rank.
  Write $\hat\Sigma = g^{-1}D\phi_\theta^\top\Lambda\,D\phi_\theta\,g^{-1}$
  for the decoder-side latent covariance,
  $q(\hat\Sigma)$ for the corresponding It\^o correction,
  $\mu_{\mathrm{dec}} := g^{-1}D\phi_\theta^\top
  [b - \tfrac{1}{2}q(\hat\Sigma)]$
  for the decoder-side drift target, and
  $\mu_{\mathrm{enc}}$ for the encoder-pullback
  target~\eqref{eq:enc-pull-target}.
  Then
  \begin{equation}\label{eq:bias-decomposition}
    \mu_{\mathrm{dec}} - \mu_{\mathrm{enc}}
    = \underbrace{(g^{-1}D\phi^\top\!-D\pi)(b-\tfrac{1}{2}q)}_{\mathrm{(I)}}
    - \tfrac{1}{2}\,\underbrace{\frobinner{\hat\Sigma}{D^2(\pi\!\circ\!\phi)}}_{\mathrm{(II)}}
    - \tfrac{1}{2}\,\underbrace{\frobinner{\Lambda-D\phi\hat\Sigma D\phi^\top}{\nabla^2\!\pi}}_{\mathrm{(III)}},
  \end{equation}
  where (I)~is the pseudo-inverse\,${}\neq{}$\,encoder gap,
  (II)~the cycle-Hessian bias, and
  (III)~the covariance mismatch.
\end{prop}

\begin{proof}
Write $\phi,\pi,g,\hat\Sigma,q$ for
$\phi_\theta,\pi_\theta,g_\theta,\hat\Sigma_\theta,q(\hat\Sigma)$,
suppressing~$\theta$.
Subtracting $\mu_{\mathrm{enc}}$ from $\mu_{\mathrm{dec}}$ and
splitting the $b$-dependent and It\^o-correction terms gives
\[
  \mu_{\mathrm{dec}} - \mu_{\mathrm{enc}}
  = (g^{-1}D\phi^\top - D\pi)\,b
    - \tfrac{1}{2}\bigl[
      g^{-1}D\phi^\top q
      + \frobinner{\Lambda}{\nabla^2\pi}
    \bigr].
\]
For the bracket, write
$g^{-1}D\phi^\top q = D\pi\,q + (g^{-1}D\phi^\top{-}D\pi)\,q$
and evaluate $D\pi\,q$ via the second-derivative chain rule
applied to $h:=\pi\circ\phi$:
\[
  D^2 h^j_{ab}
  = (\nabla^2\pi^j)_{kl}\,(D\phi)_{ka}(D\phi)_{lb}
    + (D\pi)_{jk}\,(\nabla^2\phi^k)_{ab}.
\]
Contracting with $\hat\Sigma_{ab}$ and stacking over~$j$:
\[
  \frobinner{\hat\Sigma}{D^2(\pi\circ\phi)}
  = \frobinner{D\phi\,\hat\Sigma\,D\phi^\top}{\nabla^2\pi}
    + D\pi\,q.
\]
Substituting back:
$g^{-1}D\phi^\top q + \frobinner{\Lambda}{\nabla^2\pi}
= (g^{-1}D\phi^\top{-}D\pi)\,q
  + \frobinner{\hat\Sigma}{D^2(\pi\circ\phi)}
  + \frobinner{\Lambda - D\phi\,\hat\Sigma\,D\phi^\top}{\nabla^2\pi}$.
Collecting the $(g^{-1}D\phi^\top{-}D\pi)$ terms yields the
three-term decomposition.
\end{proof}

\noindent\textbf{Interpretation.}
All three terms are biases in the decoder-side formula
for a fixed learned chart.
Terms~(I) and~(III) are first-order:
$\cL_F$ reduces Term~(I) by enforcing
$D\pi\approx g^{-1}D\phi^\top$ on $\rng(D\phi)$.
By the It\^o decomposition,
$b-\tfrac12 q(\Sigma)=D\phi\,\mu$ is exactly tangential;
when~(GC) also holds,
$\hat\Sigma\approx\Sigma$ so
$q(\hat\Sigma)\approx q(\Sigma)$ and
$b-\tfrac12 q(\hat\Sigma)$ is approximately tangential,
making $\cL_F$ effective on the relevant subspace.
The diagnostic
$\cE=\|(I-P_\theta)(b-\tfrac12 q)\|^2$
in Section~\ref{sec:experiments} measures this
normal-space leakage, while (GC) reduces Term~(III)
by aligning $\rng(\Lambda)$ with $\rng(D\phi)$.
Term~(II), the \emph{cycle-Hessian bias} governed by
$D^2(\pi\circ\phi)$, is genuinely second-order: it
requires $\pi\circ\phi$ to be locally affine, which the
first-order penalties T and F do not enforce.
Empirically, $\norm{D^2(\pi\circ\phi)}_F\approx
0.9$--$1.3$ across trained T+F models, so this term
need not be small.
Directly penalizing $\norm{D^2(\pi\circ\phi)}_F^2$
during training degrades reconstruction and tangent
alignment, because the penalty competes with the
chart-fitting objective when the autoencoder is far
from an exact inverse.
By contrast, the encoder-pullback
target~\eqref{eq:enc-pull-target} is exact by It\^o's
formula applied to the learned encoder, bypasses the
decoder-side decomposition and the amplification of
Term~(II), and in Section~\ref{sec:experiments}
consistently improves trajectory error across all
$N{\times}D$ configurations.

\subsection{From coefficient convergence to weak convergence}%
\label{sec:convergence}

The convergence chain has three links, each proved
conditionally on the preceding one.
\begin{enumerate}[leftmargin=*,nosep]
\item Chart-level generalization
  (Theorem~\ref{thm:rho-gen-same-order-rigorous}):
  the $\rho$-ERM controls reconstruction and
  tangent-bundle error.
\item Coefficient propagation
  (Theorems~\ref{thm:ambient-to-local-conv}
  and~\ref{thm:local-to-ambient-conv} below):
  assuming $\phi_n\to\phi$ in $W^{2,\infty}$ with
  $\sigmin(D\phi_n)\ge s>0$, the induced local
  coefficients $(\Sigma_n^*,\mu_n^*)$ converge
  uniformly on compacta, and the reconstructed
  ambient coefficients recover $(b,\Lambda)$.
\item Weak convergence
  (Theorem~\ref{thm:weak-convergence} below):
  given uniform coefficient convergence,
  Stroock--Varadhan yields weak convergence of the
  learned dynamics.
\end{enumerate}
The gap in this chain is between links~(1) and~(2):
the $\rho$-metric is weaker than $H^1$ and does not
control second derivatives, so $W^{2,\infty}$
convergence of the training sequence is assumed as a
sufficient condition, not derived from the training
objective.
Uniform Hessian boundedness is automatic for networks
with bounded depth and weight norms (the hypothesis
class of Assumption~\ref{ass:W1infty-approx}), but
convergence of Hessians requires additional structure
not imposed by the current penalties.

\smallskip
We first state a technical lemma on Frobenius inner
products used in both directions.

\begin{lem}\label{lem:frobinner-convergence}
Let $K\subset\Rd$ be compact, $\phi\in C^2(K;\RD)$, and
$\phi_n\in C^2(K;\RD)$ with
$\sup_n\sup_K\norm{\nabla^2\phi_n^i}_F\le C_K$ for
each $i$.
Let $A\in C^1(\Rd,S_{++}^d)$ and
$A_n\!:\Rd\to S_{++}^d$ satisfy
(a)~$A_n\to A$ uniformly on~$K$ and
(b)~$\frobinner{A}{\nabla^2\phi_n^i}
\to\frobinner{A}{\nabla^2\phi^i}$ uniformly on~$K$.
Then $\frobinner{A_n}{\nabla^2\phi_n^i}
\to\frobinner{A}{\nabla^2\phi^i}$ uniformly on~$K$.
\end{lem}

\begin{proof}
By bilinearity and an add-zero trick,
$\frobinner{A_n}{\nabla^2\phi_n^i}
-\frobinner{A}{\nabla^2\phi^i}
=\frobinner{A_n-A}{\nabla^2\phi_n^i}
+\frobinner{A}{\nabla^2\phi_n^i-\nabla^2\phi^i}$.
The second term vanishes uniformly by~(b).
By Cauchy--Schwarz the first is bounded by
$C_K\sup_K\norm{A_n-A}_F\to 0$ by~(a).
\end{proof}

\noindent\textbf{Ambient-to-local convergence.}
The next lemma controls the local covariance in terms
of the decoder and ambient covariance.

\begin{lem}\label{lem:local-cov-control}
Let $K\subset\Om$ be compact.
Let $\Sigma,\hat\Sigma$ denote local-coordinate
$d\times d$ covariance fields, and let
$\Lambda,\hat\Lambda$ denote the corresponding
ambient $D\times D$ covariance fields related via
$\Sigma=(D\phi)^\dagger\Lambda((D\phi)^\dagger)^\top$
and $\hat\Sigma=(D\hat\phi)^\dagger
\hat\Lambda((D\hat\phi)^\dagger)^\top$.
Suppose $\phi,\hat\phi$ satisfy
$\sigmin(D\phi),\sigmin(D\hat\phi)\ge s>0$ and
$\sigmax(D\phi),\sigmax(D\hat\phi)\le R$ on~$K$,
with $\norm{\Sigma}_F,\norm{\hat\Sigma}_F\le \tilde c_0$.
Then
$\norm{\Sigma-\hat\Sigma}_F
\le C_0\norm{D\phi-D\hat\phi}_F
+ C_1\norm{\Lambda-\hat\Lambda}_F$,
where
$C_0=2\tfrac{\tilde c_0\sqrt{2}}{s}
[s^{-2}+\sqrt{2}\,R^2 s^{-4}]$ and $C_1=d/s^2$.
\end{lem}

\begin{proof}
Using $\Sigma=(D\phi)^\dagger\Lambda((D\phi)^\dagger)^\top$,
expand $\Sigma-\hat\Sigma$ into a $\Lambda$-difference
term and a pseudo-inverse-difference term.
The pseudo-inverse is locally Lipschitz on
well-conditioned matrices:
$\norm{(D\phi)^\dagger-(D\hat\phi)^\dagger}_F
\le[s^{-2}+\sqrt2 R^2 s^{-4}]
\norm{D\phi-D\hat\phi}_F$, and
$\norm{(D\phi)^\dagger}\le\sqrt{d}/s$.
Substituting yields the stated constants.
\end{proof}

For the drift,
$\mu-\hat\mu = D\phi^\dagger b - D\hat\phi^\dagger\hat b
+\tfrac12(D\hat\phi^\dagger\hat q-D\phi^\dagger q)$
gives, by the same well-conditioning,
\begin{equation}\label{eq:drift-control}
  \norm{\mu-\hat\mu}
  \le \tilde C_0\norm{D\phi-D\hat\phi}_F
    + \tilde C_1\norm{b-\hat b}_2
    + \tilde C_2\norm{\hat q-q}_2.
\end{equation}

\begin{thm}\label{thm:ambient-to-local-conv}
Let $K\subset\Om$ be compact, $\phi\in C^3(\Om;\RD)$
with $\sigmin(D\phi)\ge s>0$ on~$K$,
$b\in C^1(\Om;\RD)$, $\Lambda\in C^1(\Om;S_{++}^D)$.
Denote by $(\mu,\Sigma)$ the local coefficients
induced by the chart~$\phi$ and the ambient data
$(b,\Lambda)$ via
$\mu=(D\phi)^\dagger(b-\tfrac12 q)$ and
$\Sigma=(D\phi)^\dagger\Lambda((D\phi)^\dagger)^\top$,
and analogously $(\hat\mu_n,\hat\Sigma_n)$ from
$(\hat\phi_n,\hat b_n,\hat\Lambda_n)$.
Let $\hat\phi_n$ satisfy $\sigmin(D\hat\phi_n)\ge s>0$
on~$K$ with
$\sup_n\sup_K\norm{\nabla^2\hat\phi_n^i}_F\le C_K$, and
assume
$\hat\phi_n\to\phi$ in $W^{1,\infty}(K)$,
$\hat b_n\to b$, $\hat\Lambda_n\to\Lambda$ uniformly
on~$K$, and for all bounded
$A\in C_b^1(\Om,S_{++}^d)$,
$\frobinner{A}{\nabla^2\hat\phi_n^i}\to
\frobinner{A}{\nabla^2\phi^i}$ uniformly on~$K$.
Then $\hat\mu_n\to\mu$ and $\hat\Sigma_n\to\Sigma$
uniformly on~$K$.
\end{thm}

\begin{proof}
By Lemma~\ref{lem:local-cov-control}, $D\hat\phi_n\to
D\phi$ and $\hat\Lambda_n\to\Lambda$ in $L^\infty(K)$
imply $\hat\Sigma_n\to\Sigma$.
For the drift,
Lemma~\ref{lem:frobinner-convergence} applied with
$A_n=\hat\Sigma_n$, $A=\Sigma$ and the uniform Hessian
bound gives $\hat q_n\to q$ uniformly.
Inserting into~\eqref{eq:drift-control} together with
the $W^{1,\infty}$ and $b$-convergence assumptions
yields $\hat\mu_n\to\mu$.
\end{proof}

\noindent\textbf{Local-to-ambient convergence.}
The reverse direction closes the loop: once Stage 2/3
surrogates fit the converging targets, the reconstructed
ambient coefficients recover $(b,\Lambda)$.

\begin{thm}\label{thm:local-to-ambient-conv}
Let $K\subset\Om$ be compact,
$\phi\in C^3(\Om,\RD)$ with $D\phi$ of rank~$d$,
$\mu\in C^1(\Om,\Rd)$, $\Sigma\in C^1(\Om,S_{++}^d)$,
and define
$b=D\phi\,\mu+\tfrac12 q$ and
$\Lambda=D\phi\,\Sigma D\phi^\top$.
Suppose $(\phi_n,\mu_n,\Sigma_n)$ satisfy, uniformly
on~$K$: $\phi_n\to\phi$, $D\phi_n\to D\phi$,
$\mu_n\to\mu$, $\Sigma_n\to\Sigma$, and for all
bounded $A\in C_b^1(\Om,S_{++}^d)$,
$\frobinner{A}{\nabla^2\phi_n^i}\to
\frobinner{A}{\nabla^2\phi^i}$.
Assume
$\sup_n\sup_K\norm{\nabla^2\phi_n^i}_F<\infty$ and
that $\mu_n,\Sigma_n$ are uniformly bounded on~$K$.
Setting $b_n=D\phi_n\mu_n+\tfrac12 q_n$ and
$\Lambda_n=D\phi_n\Sigma_n D\phi_n^\top$,
we have $b_n\to b$ and $\Lambda_n\to\Lambda$
uniformly on~$K$.
\end{thm}

\begin{proof}
Write
$D\phi_n\mu_n-D\phi\mu
=(D\phi_n-D\phi)\mu_n+D\phi(\mu_n-\mu)$;
both terms vanish uniformly by hypothesis.
Lemma~\ref{lem:frobinner-convergence} with
$A_n=\Sigma_n$ gives $q_n\to q$, hence $b_n\to b$.
For the covariance, the identity
$\Lambda_n-\Lambda
=D\phi_n[(\Sigma_n-\Sigma)D\phi_n^\top
+\Sigma(D\phi_n-D\phi)^\top]
+(D\phi_n-D\phi)\Sigma D\phi^\top$
with uniform bounds on $D\phi_n,\Sigma$ yields
$\Lambda_n\to\Lambda$.
\end{proof}

The following proposition bridges the two theorems
from a single joint hypothesis.

\begin{prop}\label{prop:strong-assumption-bridge}
Fix a compact $K\subset\Om$.
Suppose $\phi_n\in C^2$, $\phi\in C^3$, and
$\mu_n,\mu,\Sigma_n,\Sigma$ are $C^1$ with
$\norm{\phi_n-\phi}_{W^{2,\infty}(K)}
+\norm{\mu_n-\mu}_{L^\infty(K)}
+\norm{\Sigma_n-\Sigma}_{L^\infty(K)}\to 0$
and $\inf_n\inf_K\sigmin(D\phi_n)\ge s>0$.
Then the hypotheses of both
Theorems~\ref{thm:ambient-to-local-conv}
and~\ref{thm:local-to-ambient-conv} hold on~$K$.
In particular, $b_n\to b$ and $\Lambda_n\to\Lambda$
uniformly on~$K$.
\end{prop}

\begin{proof}
$W^{2,\infty}(K)$ convergence gives
$\nabla^2\phi_n^i\to\nabla^2\phi^i$ uniformly,
implying uniform Hessian boundedness and
Cauchy--Schwarz:
$\abs{\frobinner{A}{\nabla^2\phi_n^i-\nabla^2\phi^i}}
\le\norm{A}_{F,\infty}\norm{\nabla^2\phi_n^i
-\nabla^2\phi^i}_{F,\infty}\to 0$.
Uniform convergence on compact~$K$ gives uniform
boundedness of $\mu_n,\Sigma_n$.
Direct comparison with the hypotheses of the two
theorems completes the argument.
\end{proof}

Given link~(2) above, the following theorem
completes the convergence chain.

\begin{thm}[Weak convergence]\label{thm:weak-convergence}
Let $Z^{(n)}$ be processes solving the martingale problems for
$(\mu_n,\Sigma_n)$ on~$\Om$, and set
$X^{(n)}:=\phi_n(Z^{(n)})$.
Let $Z$ solve the martingale problem for $(\mu,\Sigma)$, and set
$X:=\phi(Z)$.
Assume there exists a compact set $K\subset \Om$ such that
\[
  \Prob\bigl(Z_t^{(n)}\in K \text{ for all } t\in[0,T]\bigr)=1,
  \qquad
  \Prob\bigl(Z_t\in K \text{ for all } t\in[0,T]\bigr)=1
\]
for every~$n$.
Assume further that $(\mu,\Sigma)$ are continuous on~$K$ with
$\Sigma$ uniformly positive definite, and that
$\mu_n\to\mu$, $\Sigma_n\to\Sigma$, and $\phi_n\to\phi$
uniformly on~$K$.
Finally, assume that for each starting point $z\in K$, the martingale
problem for $(\mu_n,\Sigma_n)$ admits a solution $\Prob_n^z$, and the
target martingale problem for $(\mu,\Sigma)$ admits a unique solution
$\Prob^z$.
Fix a common initial point $z_0\in K$ and set
$Z_0^{(n)}=Z_0=z_0$.
Then
\[
  \cL(Z^{(n)}) \Rightarrow \cL(Z)
  \quad\text{on } C([0,T];\Rd), \quad  \cL(X^{(n)}) \Rightarrow \cL(X)
  \quad\text{on } C([0,T];\RD).
\]
\end{thm}

\begin{proof}
Since $Z^{(n)}$ and $Z$ remain in~$K$ almost surely up to time~$T$,
we may extend $(\mu_n,\Sigma_n)$ and $(\mu,\Sigma)$ from~$K$ to
bounded continuous coefficients on~$\Rd$
(preserving positive semidefiniteness)
without changing the laws of the processes on~$[0,T]$;
see Appendix~\ref{app:deferred-convergence} for details.
The uniform convergence on~$K$ ensured by
the coefficient-level convergence assumptions transfers to these
extensions, so the Stroock--Varadhan theorem~\cite{stroock1979multidimensional} yields
$\Prob_n^z\Rightarrow \Prob^z$ on $C([0,T];\Rd)$.

For the ambient processes, uniform convergence $\phi_n\to\phi$ on~$K$
implies
\[
  \sup_{t\le T}\norm{\phi_n(Z_t^{(n)})-\phi(Z_t^{(n)})}_2
  \le \sup_{x\in K}\norm{\phi_n(x)-\phi(x)}_2
  \longrightarrow 0
\]
deterministically.
Applying the continuous mapping theorem to the fixed continuous map
$\gamma\mapsto \phi\circ\gamma$ and combining with the display above
gives $\phi_n(Z^{(n)})\Rightarrow \phi(Z)$ on $C([0,T];\RD)$.
\end{proof}

\begin{cor}[MFPT convergence]\label{cor:mfpt-convergence}
  Under the hypotheses of Theorem~\ref{thm:weak-convergence},
  let $\tau_r^{(n)}=\inf\{t\ge 0:\norm{X_t^{(n)}-X_0}_2\ge r\}\wedge T$
  and $\tau_r$ the corresponding hitting time for~$X$,
  both capped at~$T$.
  Assume the limiting process crosses the sphere
  $\{\norm{x-X_0}=r\}$ transversally almost surely:
  writing
  $\tau_r^+=\inf\{t:\norm{X_t-X_0}>r\}\wedge T$,
  \[
    \Prob\bigl(\tau_r < T \;\text{and}\; \tau_r = \tau_r^+\bigr)=1.
  \]
  Then $\tau_r^{(n)}\Rightarrow\tau_r$ in distribution
  and $\E[\tau_r^{(n)}]\to\E[\tau_r]$.
\end{cor}

\begin{proof}
  The hitting-time functional
  $\gamma\mapsto\inf\{t:\norm{\gamma(t)-\gamma(0)}\ge r\}\wedge T$
  is continuous at every path that crosses the sphere
  immediately (i.e.\ $\tau_r=\tau_r^+<T$); see
  \cite[Thm.~13.6.1]{whitt2002stochastic}.
  By the transversality assumption the limit~$X$ lies a.s.\ in the
  continuity set, so the continuous mapping theorem gives
  $\tau_r^{(n)}\Rightarrow\tau_r$.
  Since $\tau_r^{(n)}\in[0,T]$, the family is uniformly bounded
  and hence uniformly integrable; convergence of means follows.
\end{proof}


\section{Experiments}\label{sec:experiments}

The convergence chain of
Sections~\ref{sec:rho-metric}--\ref{sec:generalization}
has three links:
(i)~$\rho$-metric chart quality
$\to$ (ii)~coefficient convergence
$\to$ (iii)~weak/MFPT convergence.
This paper controls link~(i) via
Theorem~\ref{thm:rho-gen-same-order-rigorous};
link~(ii) is established in
Theorems~\ref{thm:ambient-to-local-conv}
and~\ref{thm:local-to-ambient-conv} under a
$W^{2,\infty}$ chart-convergence assumption
(Proposition~\ref{prop:strong-assumption-bridge});
link~(iii) is Theorem~\ref{thm:weak-convergence} and
Corollary~\ref{cor:mfpt-convergence}.
We validate each link experimentally.
The ablation study in Section~\ref{ssec:ablation} measures
chart quality via reconstruction error, tangent-space error,
and tangent-space fidelity~$\cE$, as well as end-to-end
coefficient accuracy via $\cE_b$ and $\cE_\Lambda$.
Section~\ref{ssec:mfpt} tests the full chain from
chart to dynamics via MFPT under both smooth and
metastable regimes.
Post-training diagnostics verify the well-conditioning
assumption $\sigmin(D\phi)\ge s$
of Assumption~\ref{ass:target-chart}.

\subsection{Setup}\label{ssec:setup}\label{ssec:surfaces}

\noindent \textbf{Surfaces and embedding.}
We evaluate on four Monge-patch surfaces
$\phi(u,v)=(u,\,v,\,f(u,v))^\top$ with $(u,v)\in[-1,1]^2$:
paraboloid ($f=u^2+v^2$, positive Gaussian curvature),
hyperbolic paraboloid ($f=u^2-v^2$, negative curvature),
quartic dome ($f=(u^2{+}v^2)-(u^4{+}v^4)/2$, sign-changing curvature),
sinusoidal ($f=\sin(u+v)$, intrinsically flat).
Each surface is embedded in $\R^D$ by appending $K_F$
Fourier coordinate pairs, giving $D=3+2K_F$;
the Monge-patch bound $\sigmin(D\phi)\ge 1$ is preserved.
We test $D\in\{11,201\}$.

\noindent \textbf{Two dynamics.}
We consider two latent SDEs with complementary
challenges.

\emph{(i) Overdamped Langevin (MB).}
The M\"uller--Brown potential~\cite{MullerBrown1979}
is rescaled to $(u,v)\in[-1,1]^2$ via
$x=2.25\,u-0.25$ and $y=2.25\,v+1.0$, divided by
$V_0{=}200$, giving drift
$\mu=-\nabla V_{\mathrm{MB}}$ and isotropic diffusion
$\sigma=\sqrt{2k_BT}\,I_2$ with $k_BT{=}0.10$.
The rescaled potential has three metastable wells and
two saddle-point transition channels; this tests
whether the learned SDE captures inter-well
transition rates.

\emph{(ii) Rotation drift with state-dependent
diffusion (Rot).}
The drift and diffusion are
$\mu(u,v)=(-v,\,u)^\top$ and
\[
  \sigma(u,v) =
  \begin{pmatrix}
    1+u^2/4 & u+v \\
    0 & 1+v^2/4
  \end{pmatrix}.
\]
This exercises all three pipeline stages with
anisotropic, position-dependent noise.
Full SDE parameters and MB coefficients are in
Appendix~\ref{app:mb-coeffs}.

\noindent \textbf{Training and penalties.}
All experiments use $10$ random seeds and follow
Algorithm~\ref{alg:pipeline}.
Rotation uses $N{=}50$ training points, the sparse-data
regime where geometric penalties matter most; MB uses
$N{=}200$, enough to populate the three metastable wells.
The ablation study in Section~\ref{ssec:ablation} compares
six conditions:
\emph{baseline} (reconstruction loss only),
\emph{T} (tangent-bundle penalty~$\cL_T$ from
Section~\ref{ssec:tangent-from-cov}, $\lambda_T{=}1$),
\emph{F} (inverse-consistency penalty
$\cL_F{=}\|D\pi\,D\phi{-}I_d\|_F^2$, $\lambda_F{=}1$),
\emph{C} (contractive penalty~\cite{contractiveAutoEncoders},
$\cL_C{=}\|D\pi\|_F^2$, $\lambda_C{=}0.01$),
\emph{T+F} (both $\cL_T$ and $\cL_F$), and
\emph{ATLAS}~\cite{yym} (the local-chart baseline of
Section~\ref{sec:intro}, using Gaussian-kernel blending
of drift and diffusion at $N$ landmarks with oracle
coefficients).
Architecture details are in Appendix~\ref{app:architecture}.

\noindent \textbf{Oracle coefficients.}
All methods receive exact ambient drift $b(x_i)$ and
covariance $\Lambda(x_i)$ computed from the known SDE,
isolating the effect of chart-quality regularization
from coefficient-estimation noise.
In practice, $(b,\Lambda)$ would be estimated from
short-burst trajectory data~\cite{Miles17,yym};
quantifying the resulting perturbation to the
pipeline losses is an open problem.

\noindent \textbf{Landmarks and baselines.}
Landmarks are placed via a greedy $\delta$-net in
the induced Riemannian metric: candidates are accepted
greedily so that consecutive landmarks are separated
by at least~$\delta$, with $\delta$ chosen by binary
search to yield approximately~$N$ points.
This quasi-uniform design follows the ATLAS
construction~\cite{Miles17} and provides good
geometric coverage at a given landmark budget.
ATLAS is included in the coefficient ablation but
excluded from MFPT tables because its ambient-space
simulation becomes infeasible at $D{=}201$ and its
local-chart interpolation degrades under sparse
landmarks.

\noindent \textbf{Statistical tests.}
Unless otherwise noted, significance is assessed by
one-sided paired Wilcoxon signed-rank tests for
directional AE-vs-baseline comparisons.
Table~\ref{tab:rotation_mfpt} (rotation MFPT) uses
paired $t$-tests with common Brownian noise and
reports means.
No multiple-comparison correction is applied; marginal
$p$-values near $0.05$ in Table~\ref{tab:mb_mfpt}
should be interpreted as directional evidence.

\noindent \textbf{Metrics.}
We evaluate at five levels:
(i)~\emph{chart quality}: reconstruction error
$\|x-\phi(\pi(x))\|^2$ and tangent-space error
$\|P_\theta-P\|_F^2$ on $500$ held-out test points,
reported as per-seed medians;
(ii)~\emph{tangent-space fidelity}:
$\cE := \|(I - P_\theta)(b - \tfrac{1}{2}q)\|^2$,
the normal-space residual of the It\^o-corrected drift;
(iii)~\emph{end-to-end coefficient quality}:
$\cE_b := \|D\phi\,\hat\mu + \tfrac{1}{2}q(\hat\sigma\hat\sigma^\top) - b\|^2$
(ambient drift error) and
$\cE_\Lambda := \|D\phi\,\hat\sigma\hat\sigma^\top D\phi^\top - \Lambda\|_F^2$
(ambient covariance error), reported as per-seed
medians over $200$ held-out evaluation points;
(iv)~\emph{trajectory quality}: we use two MFPT
observables.
Under rotation drift, \emph{radial MFPT} at radius~$r$
measures the mean first-passage time to ambient
distance~$r$ from the start, averaged over $500$
trajectories with $T{=}2$ and $\Delta t{=}0.01$;
under the transversality condition of
Corollary~\ref{cor:mfpt-convergence}, convergence of
means follows from weak convergence.
Under MB Langevin, we report \emph{per-pair inter-well
MFPT} for the W0${\to}$W1 and W0${\to}$W2 transitions,
using $2000$ trajectories with $T{=}50$ and
$\Delta t{=}0.005$.
The learned and ground-truth processes are driven by
the same Brownian increments (common random numbers)
so that the reported difference isolates model error
from Monte~Carlo variance.
By the Kramers rate~\cite{Kramers1940,Hanggi1990}
$k\sim\exp(-\Delta V/k_BT)$, inter-well MFPT is
exponentially sensitive to drift errors in the
barrier region, making it a stringent empirical
stress test beyond the fixed-horizon guarantee of
Theorem~\ref{thm:weak-convergence}.
Both MFPT metrics are reported as relative error
vs.\ ground-truth simulation on the true manifold;
further algorithmic details (core-set well assignment,
dwell confirmation, censoring) are in
Appendix~\ref{app:mfpt}.
(v)~\emph{Extrapolation}: reconstruction error at
increasing distance~$\delta$ beyond the $[-1,1]^2$
training domain, evaluated on $10$~seeds at $D{=}11$.

\subsection{Chart and coefficient ablation}\label{ssec:ablation}

Table~\ref{tab:ablation} reports chart quality,
measured by reconstruction error, tangent error, and
$\cE$ on the
paraboloid at $D{=}11$ and $D{=}201$ under both dynamics
with $10$~seeds each; full results on all four surfaces are
in Appendix~\ref{app:full-ablation}.

\begin{table}[ht]
\centering
\caption{Ablation study on the paraboloid, $10$~seeds,
medians reported. \textbf{Bold}~= best per column.
${}^{**}p{<}0.01$ vs.\ baseline, paired Wilcoxon test.
$\cE$: tangent-space fidelity;
$\cE_b$, $\cE_\Lambda$: true end-to-end ambient
drift and covariance errors from the full pipeline.
F and C coincide at MB $D{=}201$: without tangent
alignment, both collapse across all $10$~seeds to a
degenerate chart with near-zero decoder-Jacobian
singular values, underscoring the need for the
tangent-bundle penalty.}
\label{tab:ablation}
\footnotesize
\begin{tabular}{l ccccc ccccc}
\toprule
\multicolumn{11}{c}{\emph{Rotation drift} ($N{=}50$)} \\
\cmidrule(lr){1-11}
 & \multicolumn{5}{c}{$D{=}11$} & \multicolumn{5}{c}{$D{=}201$} \\
\cmidrule(lr){2-6} \cmidrule(lr){7-11}
 & Rec.\ & Tang.\ & $\cE$ & $\cE_b$ & $\cE_\Lambda$
 & Rec.\ & Tang.\ & $\cE$ & $\cE_b$ & $\cE_\Lambda$ \\
\midrule
ATLAS    & $.008$ & $.185$ & $.104$ & $\mathbf{.10}$ & $2.95$
         & $.001$ & $.139$ & $\mathbf{.140}$ & $\mathbf{.14}$ & $3.15$ \\
baseline & $.001$ & $.029$ & $.536$ & $3.67$ & $4.54$
         & ${\scriptstyle <}.001$ & $.093$ & $1.67$ & $7.54$ & $5.79$ \\
\textbf{T}$^{**}$ & $\mathbf{{\scriptstyle <}.001}$ & $\mathbf{.002}$ & $\mathbf{.032}$ & $3.51$ & $1.48$
         & $\mathbf{{\scriptstyle <}.001}$ & $\mathbf{.004}$ & $.320$ & $9.35$ & $\mathbf{2.34}$ \\
F        & $.017$ & $.490$ & $5.60$ & $5.46$ & $22.7$
         & $.006$ & $2.00$ & $11.7$ & $11.7$ & $53.5$ \\
C        & $.001$ & $.055$ & $.736$ & $2.62$ & $6.27$
         & $.006$ & $1.64$ & $10.2$ & $12.2$ & $48.6$ \\
\textbf{T+F}$^{**}$ & $.001$ & $.005$ & $.063$ & $.41$ & $\mathbf{1.19}$
         & $.001$ & $.035$ & $.540$ & $1.86$ & $3.83$ \\
\midrule
\multicolumn{11}{c}{\emph{MB Langevin} ($N{=}200$)} \\
\cmidrule(lr){1-11}
 & \multicolumn{5}{c}{$D{=}11$} & \multicolumn{5}{c}{$D{=}201$} \\
\cmidrule(lr){2-6} \cmidrule(lr){7-11}
 & Rec.\ & Tang.\ & $\cE$ & $\cE_b$ & $\cE_\Lambda$
 & Rec.\ & Tang.\ & $\cE$ & $\cE_b$ & $\cE_\Lambda$ \\
\midrule
ATLAS    & $.001$ & $.012$ & $.773$ & $.465$ & ${\scriptstyle <}.001$
         & ${\scriptstyle <}.001$ & $.014$ & $.779$ & $.511$ & ${\scriptstyle <}.001$ \\
baseline & ${\scriptstyle <}.001$ & $.003$ & $.006$ & $.029$ & $.001$
         & ${\scriptstyle <}.001$ & $.007$ & $.014$ & $.040$ & ${\scriptstyle <}.001$ \\
\textbf{T}$^{**}$ & $\mathbf{{\scriptstyle <}.001}$ & $\mathbf{{\scriptstyle <}.001}$ & $\mathbf{{\scriptstyle <}.001}$ & $.017$ & $\mathbf{{\scriptstyle <}.001}$
         & $\mathbf{{\scriptstyle <}.001}$ & $\mathbf{{\scriptstyle <}.001}$ & $\mathbf{{\scriptstyle <}.001}$ & $.024$ & $\mathbf{{\scriptstyle <}.001}$ \\
F        & ${\scriptstyle <}.001$ & $.015$ & $.039$ & $.088$ & $.006$
         & $.002$ & $2.00$ & $11.3$ & $10.8$ & $.426$ \\
C        & $.002$ & $.430$ & $1.64$ & $3.95$ & $.062$
         & $.002$ & $2.00$ & $11.3$ & $10.8$ & $.426$ \\
\textbf{T+F}$^{**}$ & $\mathbf{{\scriptstyle <}.001}$ & $\mathbf{{\scriptstyle <}.001}$ & $\mathbf{{\scriptstyle <}.001}$ & $\mathbf{.011}$ & $\mathbf{{\scriptstyle <}.001}$
         & $\mathbf{{\scriptstyle <}.001}$ & $\mathbf{{\scriptstyle <}.001}$ & $\mathbf{{\scriptstyle <}.001}$ & $\mathbf{.022}$ & $\mathbf{{\scriptstyle <}.001}$ \\
\bottomrule
\end{tabular}
\end{table}

The tangent-bundle penalty~T directly minimizes the
projector term of the $\rho$-metric
(Section~\ref{ssec:tangent-from-cov}).
At $D{=}11$, T alone reduces tangent error $15{\times}$
(from $.029$ to $.002$) and $\cE$ by $17{\times}$.
Under MB~Langevin ($N{=}200$), T and T+F both achieve
$\cE < 10^{-3}$, while baseline sits at $.006$ and
ATLAS at $.77$, two to three orders of magnitude worse.
Adding~F yields the best reconstruction, extrapolation
(Figure~\ref{fig:extrapolation}), and end-to-end drift
accuracy $\cE_b$ ($0.41$ vs T's $3.51$ at $D{=}11$),
because the encoder-pullback target requires the
inverse consistency that F enforces.

F alone and~C penalize $D\pi$ without a
dynamics-informed tangent target, so they do not reduce
the projector term that drives the $\rho$-metric.
Under MB at $D{=}201$, both collapse to a degenerate AE
with $\cE > 10$ and tangent error saturating at $2.0$,
indicating that the learned tangent space is essentially
unrelated to the true one.
ATLAS~\cite{yym} interpolates oracle coefficients via
Gaussian kernels and achieves the lowest $\cE_b$ under
rotation ($0.10$--$0.14$), confirming that its
coefficient approximation is locally accurate.
However, its tangent-space fidelity is limited by
kernel smoothing: $\cE{=}0.15$ under rotation
(vs $0.053$ for T) and $\cE{=}0.77$ under MB
(vs $<10^{-3}$ for T+F).
The three computational bottlenecks from
Section~\ref{sec:intro} persist: exponential landmark
scaling, per-step re-projection onto~$M$, and
$O(mD^2)$ simulation cost, making ATLAS infeasible
at $D{=}201$.

The columns $\cE_b$ and $\cE_\Lambda$ report the true
ambient drift and covariance errors of the full
three-stage pipeline.
Under rotation, T+F reduces $\cE_b$ by
$9{\times}$ at $D{=}11$ and $4{\times}$ at $D{=}201$
relative to baseline.
T alone achieves good $\cE_\Lambda$ but
worse $\cE_b$ than baseline, because the
encoder-pullback drift target assumes $D\pi\,D\phi\approx I_d$;
without F enforcing this, the targets are biased and the
drift network learns biased coefficients.
This validates the complementary roles of T and~F:
T aligns the tangent space, F ensures the encoder--decoder
consistency that the pullback formula requires.
Under MB, all AE conditions achieve very small $\cE_b$
and $\cE_\Lambda$, with T+F yielding the best drift
accuracy ($\cE_b{=}.011$ at $D{=}11$).
Post-training diagnostics confirm that
$\sigmin(D\phi_\theta)\ge s>0$ holds across all
T+F models, validating the hypothesis of
Assumption~\ref{ass:target-chart}.

Figure~\ref{fig:extrapolation} shows that
T+F consistently extrapolates best, with error growing
$2$--$3{\times}$ slower than baseline at $\delta{=}0.3$.
F alone produces the worst extrapolation: without
tangent alignment, the decoder diverges off-manifold.

\begin{figure}[ht]
\centering
\includegraphics[width=\textwidth]{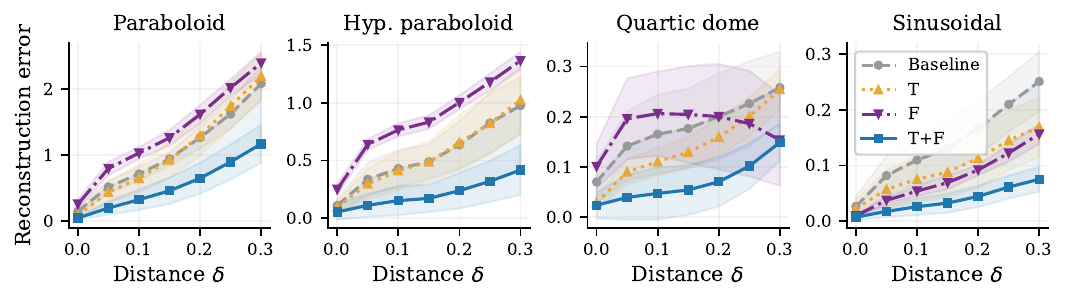}
\caption{Reconstruction error vs extrapolation distance~$\delta$
beyond the $[-1,1]^2$ training domain ($D{=}11$, $10$~seeds,
mean $\pm$ std).
T+F extrapolates best on all four surfaces.}
\label{fig:extrapolation}
\end{figure}

\subsection{From coefficients to dynamics}\label{ssec:mfpt}\label{ssec:mb_mfpt}\label{ssec:radial_mfpt}

We now test whether improved chart and coefficient
quality translates to better dynamics.
The ablation above shows that F~alone and~C fail to
learn a meaningful tangent space, and ATLAS cannot
simulate at $D{=}201$, so the dynamics
comparison focuses on baseline, T, and T+F.

\paragraph{Rotation drift}
Table~\ref{tab:rotation_mfpt} reports radial MFPT
relative error at $r{=}2$ across all four surfaces.
T+F reduces error by $50$--$70\%$ relative to baseline
and is significant ($p{<}0.05$) on all surfaces at both
dimensions, with T~alone showing intermediate gains at
$D{=}201$.
Radial MFPT is a smooth functional with no barrier
sensitivity, so the $\cE_b$ improvement from the
ablation translates directly to dynamics, consistent
with Corollary~\ref{cor:mfpt-convergence}.

\begin{table}[ht]
\centering
\caption{Radial MFPT relative error~(\%) at $r{=}2$
under rotation drift ($N{=}50$, $\delta$-net landmarks,
$10$~seeds, mean).
${}^{**}p{<}0.01$, ${}^{*}p{<}0.05$ vs.\ baseline,
paired $t$-test with common noise.}
\label{tab:rotation_mfpt}
\small
\begin{tabular}{l ccc ccc}
\toprule
 & \multicolumn{3}{c}{$D{=}11$} & \multicolumn{3}{c}{$D{=}201$} \\
\cmidrule(lr){2-4} \cmidrule(lr){5-7}
Surface & base & T & \textbf{T+F} & base & T & \textbf{T+F} \\
\midrule
Paraboloid   & $39.8$ & $37.6$ & $\mathbf{16.0}^{*}$
             & $38.3$ & $38.9$ & $\mathbf{19.2}^{**}$ \\
Hyp.\ parab. & $49.9$ & $32.2^{*}$ & $\mathbf{18.7}^{**}$
             & $51.9$ & $32.3^{*}$ & $\mathbf{20.6}^{**}$ \\
Quartic dome & $25.7$ & $27.1$ & $\mathbf{11.0}^{**}$
             & $29.3$ & $14.8^{**}$ & $\mathbf{12.1}^{**}$ \\
Sinusoidal   & $16.9$ & $16.1$ & $\mathbf{6.0}^{**}$
             & $25.9$ & $19.3^{**}$ & $\mathbf{7.7}^{**}$ \\
\bottomrule
\end{tabular}
\end{table}

\paragraph{Metastable dynamics}
Inter-well MFPT goes beyond the fixed-horizon weak
convergence of Theorem~\ref{thm:weak-convergence}:
the Kramers rate~\cite{Kramers1940,Hanggi1990}
$k\sim\exp(-\Delta V/k_BT)$ makes it exponentially
sensitive to drift errors in the barrier region.
We simulate $2000$ trajectories starting near well~W0
with paired Brownian noise to isolate model error from
Monte~Carlo variance and report per-pair MFPT for the
W0${\to}$W1 and W0${\to}$W2 transitions; see the
Appendix~\ref{app:mfpt} for details.

Table~\ref{tab:mb_mfpt} shows that T+F achieves the
lowest MFPT error on most surface--transition pairs,
though several $p$-values are near $0.05$ and no
multiple-comparison correction is applied, so these
results should be interpreted as directional evidence
rather than definitive.
On the paraboloid, W0${\to}$W1 error drops from
$5.1\%$ to $2.2\%$ at $D{=}11$ ($p{<}0.05$) and
W0${\to}$W2 from $4.3\%$ to $1.7\%$ at $D{=}201$
($p{<}0.05$).
On the hyperbolic paraboloid at $D{=}201$, T+F
reduces W0${\to}$W1 from $6.6\%$ to $4.1\%$
($p{<}0.05$) and T alone reaches $2.9\%$ on
W0${\to}$W2.
Quartic dome and sinusoidal results are directionally
consistent, with T+F achieving the lowest error on
most transitions and all conditions in the $1$--$5\%$
range.
Preliminary oracle-drift experiments (replacing the
learned Stage~2 with exact latent drift) suggest that,
at $D{=}201$, the remaining error is limited by
drift learning near the saddle rather than by
autoencoder quality.

\begin{table}[ht]
\centering
\caption{Per-pair MFPT error~(\%) under MB Langevin
($N{=}200$, $\delta$-net landmarks, $10$~seeds, median).
Only W0${\to}$W1 and W0${\to}$W2 reported, since
all trajectories start at W0.
\textbf{Bold}~= lowest error per row.
${}^{*}p{<}0.05$ vs.\ baseline,
one-sided paired Wilcoxon with common noise.}
\label{tab:mb_mfpt}
\small
\begin{tabular}{ll ccc ccc}
\toprule
 & & \multicolumn{3}{c}{$D{=}11$} & \multicolumn{3}{c}{$D{=}201$} \\
\cmidrule(lr){3-5} \cmidrule(lr){6-8}
Surface & Pair & base & T & T+F & base & T & T+F \\
\midrule
Paraboloid & W0${\to}$W1
  & $5.1$ & $4.4^{*}$ & $\mathbf{2.2}^{*}$
  & $4.1$ & $4.9$ & $\mathbf{3.2}$ \\
 & W0${\to}$W2
  & $2.7$ & $\mathbf{1.3}$ & $2.2$
  & $4.3$ & $2.6$ & $\mathbf{1.7}^{*}$ \\
\midrule
Hyp.\ parab. & W0${\to}$W1
  & $5.3$ & $4.8$ & $\mathbf{4.3}$
  & ${6.6}$ & $4.8$ & $\mathbf{4.1}^*$ \\
 & W0${\to}$W2
  & $3.3$ & $\mathbf{2.1}$ & $2.7$
  & $6.8$ & $\mathbf{2.9}$ & $4.0^{*}$ \\  
\midrule
Quartic dome & W0${\to}$W1
  & ${4.3}$ & $4.5$ & $\mathbf{3.7}$
  & ${5.5}$ & $4.9$ & $\mathbf{4.3}$ \\
 & W0${\to}$W2
  & $2.1$ & $2.3$ & $\mathbf{1.4}^*$
  & ${4.4}$ & $4.3$ & $\mathbf{2.9}$ \\
\midrule
Sinusoidal & W0${\to}$W1
  & $2.5$ & $4.3$ & $\mathbf{2.1}$
  & $2.8$ & $2.9$ & $\mathbf{2.3}$ \\
 & W0${\to}$W2
  & $3.1$ & $1.8$ & $\mathbf{1.2}^*$
  & $2.5$ & $2.1$ & $\mathbf{1.2}^*$ \\
\bottomrule
\end{tabular}
\end{table}

\section{Discussion and Future Work}\label{sec:discussion}

We have developed a geometric regularization framework for autoencoders trained on
data from non-singular diffusions on Riemannian submanifolds.
The $\rho$-metric controls first-order geometry through the tangent-bundle penalty,
achieves the same generalization rate as Sobolev $H^1$ training
for chart quality in an oracle $\rho$-ERM
(Theorem~\ref{thm:rho-gen-same-order-rigorous}),
and under stronger $W^{2,\infty}$ regularity, chart-level errors
propagate controllably to the ambient SDE coefficients
as shown in Section~\ref{sec:convergence}.
The encoder-pullback drift target~\eqref{eq:enc-pull-target}
provides exact latent drift via It\^o's formula applied to the
learned encoder, avoiding the systematic bias of the decoder-side
formula (Proposition~\ref{prop:dec-bias}).

The two penalties T and F serve complementary roles.
T is chart-invariant by construction, built from $\Lambda$
(Lemma~\ref{lem:coord-invariance}), and controls tangent-space
orientation.
F enforces the coordinate-invariant condition
$D\pi\,D\phi = I_d$; the Frobenius penalty
$\norm{D\pi\,D\phi-I_d}_F^2$ itself is not invariant
under reparametrization, but this is moot in practice
because the latent coordinates are learned jointly and
only one chart is ever trained.
Together, T+F produces charts with well-aligned tangent bundles and
accurate encoder--decoder inverses, which determine the quality of
the induced latent SDE.
To assess the downstream effect on dynamics, we use MFPT,
which depends only on the law of the learned process and
avoids the coupling dependence of pathwise metrics.
Under rotation drift, T+F translates directly to better
dynamics on all four surfaces.
Under metastable Langevin dynamics, T+F achieves the
lowest MFPT error on most surface--transition pairs.
T+F also extrapolates $2$--$3{\times}$ better than
baseline beyond the training domain
(Figure~\ref{fig:extrapolation}), a prerequisite for
multi-chart extensions.

\noindent \textbf{Future work.}
The single-chart theory extends to a finite atlas via
chart-wise penalties weighted by a partition of unity,
requiring control of transition-map Lipschitz constants.
The encoder Hessian gap not targeted by T+F could be
addressed by second-order drift objectives such as
penalizing pullback-drift
roughness~\cite{Bittracher2018} or operator-aligned
regression preserving dominant
timescales~\cite{Klus2018}.
Finally, replacing oracle $(b,\Lambda)$ with estimates
bootstrapped from short trajectory
bursts~\cite{Miles17,yym} and quantifying the resulting
perturbation to the $\rho$-loss is an important open
problem.

\appendix

\section{Deferred proofs}\label{app:deferred}

\subsection{Error propagation: additional details}\label{app:deferred-convergence}

The coefficient-propagation theorems and bridging
proposition are stated and proved in the main text
(Section~\ref{sec:convergence}).
We record here only the discussion of the
$W^{2,\infty}$ assumption and the pathwise $L^p$
error bound that complements the weak-convergence
result.

\noindent \textbf{Discussion of the $W^{2,\infty}$ assumption.}
The bridging proposition requires $\phi_n\to\phi$ in $W^{2,\infty}$.
Theorem~\ref{thm:rho-gen-same-order-rigorous} controls the $\rho$-risk
(reconstruction~$+$ tangent-bundle error), which lies strictly
between $L^2$ and $H^1$; it does not imply $W^{1,\infty}$ or
$W^{2,\infty}$ convergence.
The additional second-order control is needed because the Stage~2
drift target involves
$\langle\Sigma,\nabla^2\phi\rangle$, so uniform convergence of
the decoder Hessian is required to pass to the limit.
Three remarks are in order.
\begin{enumerate}[leftmargin=*]
\item \emph{Uniform Hessian boundedness.}
  For feed-forward networks with bounded weights and bounded
  activation derivatives (true of tanh, sigmoid, and softplus),
  all second derivatives of $\phi_\theta$ are automatically
  bounded by a constant depending only on the weight norms and
  depth.
  Thus the uniform Hessian bound
  $\sup_n\sup_K\|\nabla^2\hat\phi_n^i\|_F\le C_K$
  holds whenever the depth and weight-norm bounds are
  uniform in~$n$ (as is the case for the hypothesis class
  of Assumption~\ref{ass:rho-class}).
\item \emph{$W^{2,\infty}$ convergence.}
  Uniform Hessian boundedness does not imply that Hessians
  \emph{converge}.
  In general, the $\rho$-loss minimizers $\phi_n$ need not
  converge in $W^{2,\infty}$ even if they converge in
  $W^{1,\infty}$.
  We therefore present $W^{2,\infty}$ convergence as a
  \emph{sufficient condition} for the propagation chain,
  not as a consequence of the training objective.
\item \emph{Modular design.}
  The convergence pipeline is deliberately modular:
  Theorem~4.5 controls chart quality (link~(i)),
  the bridging proposition controls coefficient
  propagation under $W^{2,\infty}$ (link~(ii)),
  and Stroock--Varadhan gives weak convergence (link~(iii)).
  Closing the gap between links~(i) and~(ii)---showing that
  $\rho$-minimizers enjoy $W^{2,\infty}$ convergence, or
  finding a weaker sufficient condition---is an open problem
  noted in Section~\ref{sec:convergence}.
\end{enumerate}

\noindent \textbf{Weak convergence via Stroock--Varadhan.}
\label{ssec:weak-convergence}

The classical Stroock--Varadhan theorem~\cite{stroock1979multidimensional} provides weak convergence of
solutions to the martingale problem under uniform convergence of
coefficients.

\begin{thm}[Stroock--Varadhan]\label{thm:stroock-varadhan}
Let $\Lambda:\RD\to S_+^D$ and $b:\RD\to\RD$ be locally bounded
measurable functions which are continuous in~$x$, and assume that for
each $x\in\RD$ the martingale problem for $(b,\Lambda)$ starting
from~$x$ has exactly one solution $\Prob^x$.  Suppose that for each
$n\ge 1$, $\Lambda_n:\RD\to S_+^D$ and $b_n:\RD\to\RD$ are
measurable functions satisfying, for all $R>0$:
\[
  \sup_{n\ge 1}\sup_{\norm{x}\le R}
  \big[\norm{b_n(x)}_2+\norm{\Lambda_n(x)}_F\big]<\infty
\]
and
\[
  \lim_{n\to\infty}\sup_{\norm{x}\le R}
  \big(\norm{\Lambda_n(x)-\Lambda(x)}_F
      +\norm{b_n(x)-b(x)}_2\big) = 0.
\]
Let $\Prob_n^x$ be a solution to the martingale problem for
$(b_n,\Lambda_n)$ starting from~$x$.  Then
$\Prob_n^x\to\Prob^x$ weakly as measures on path space.
\end{thm}

\noindent\textbf{Coefficient extension.}
To apply Theorem~\ref{thm:stroock-varadhan}, it suffices to
extend $(\mu_n,\Sigma_n)$ and $(\mu,\Sigma)$ from~$K$ to
bounded continuous coefficients on~$\Rd$; Stroock--Varadhan
requires only continuity and local boundedness.
A bounded linear extension operator
$E\!:C(K)\to C_b(\Rd)$ extends each field, and
composing the matrix extensions with the nearest-point
projection onto $S_+^d$ preserves positive
semidefiniteness and compact-uniform convergence.
Since the processes remain in~$K$ a.s., their laws are
unchanged.

The proof of the weak convergence theorem appears in
Section~\ref{sec:convergence}.
The pathwise error bound below complements the weak convergence
result; its proof uses standard BDG and Gronwall tools.

\begin{thm}[$L^p$ pathwise error bound]\label{thm:pathwise-error}
Let $X_t$ and $\hat X_t$ be strong solutions of
$dX=\mu(X)\,dt+\sigma(X)\,dW$ and
$d\hat X=\hat\mu(\hat X)\,dt+\hat\sigma(\hat X)\,dW$
driven by the same Brownian motion~$W$,
with $X_0=\hat X_0$ and
$\sigma=\Sigma^{1/2}$,
$\hat\sigma=\hat\Sigma^{1/2}$, both principal square roots.
Assume there exists a compact set~$\cK\subset\Rd$ such that
\[
  \Prob\bigl(X_s,\hat X_s\in\cK \text{ for all } s\in[0,T]\bigr)=1.
\]
Suppose
\[
  \norm{\hat\mu-\mu}_{L^\infty(\cK)}\le\epsilon_\mu,
  \qquad
  \norm{\hat\Sigma-\Sigma}_{L^\infty(\cK)}\le\epsilon_\Sigma.
\]
Assume $\mu,\hat\mu$ are $C_\mu$-Lipschitz and
$\sigma,\hat\sigma$ are $C_\sigma$-Lipschitz on~$\cK$.
Assume moreover that there exists $\lambda_0>0$ such that
\[
  \lambda_{\min}(\Sigma(x)),\ \lambda_{\min}(\hat\Sigma(x))
  \ge \lambda_0
  \qquad\text{for all } x\in\cK.
\]
Then for any $p\ge 2$ and $0\le t\le T$, there exists
$C_2=C_2(p,d,T,C_\mu,C_\sigma)$ such that
\begin{equation}\label{eq:pathwise-error-bound}
  \E\big[(\Delta_t^*)^p\big]
  \le C_2\bigg[\epsilon_\mu^p
    + \frac{1}{2^p\lambda_0^{p/2}}\,\epsilon_\Sigma^p\bigg],
\end{equation}
where $\Delta_t^*=\sup_{s\le t}\norm{X_s-\hat X_s}$.
\end{thm}

\begin{proof}
Write $\Delta_s = X_s - \hat X_s$.
The BDG inequality~\cite{RogersWilliams2000} and
H\"{o}lder's inequality give, for $C=C(p,d,T)$,
\begin{equation}\label{eq:BDG-application}
  \E\big[(\Delta_t^*)^p\big]
  \le C\int_0^t
    \E\Big[\norm{\hat\mu(\hat X_s)-\mu(X_s)}^p
        +\norm{\hat\sigma(\hat X_s)-\sigma(X_s)}_F^p\Big]ds.
\end{equation}
The drift and diffusion coefficient errors decompose via triangle
inequality into approximation error plus Lipschitz feedback:
\[
  \norm{\hat\mu(\hat x)-\mu(x)}
  \le \epsilon_\mu + C_\mu\norm{\hat x - x},
  \qquad
  \norm{\hat\sigma(\hat x)-\sigma(x)}_F
  \le \tfrac{\epsilon_\Sigma}{2\sqrt{\lambda_0}}
    + C_\sigma\norm{\hat x - x},
\]
where $C_\mu,C_\sigma$ are Lipschitz constants on~$\cK$ and the
diffusion bound uses the perturbation inequality for
positive-definite square roots with eigenvalues $\ge\lambda_0$.
Substituting into~\eqref{eq:BDG-application} and applying
$(a+b)^p\le 2^{p-1}(a^p+b^p)$ yields
\[
  \E\big[(\Delta_t^*)^p\big]
  \le C'\int_0^t \E\big[(\Delta_s^*)^p\big]\,ds
    + C''\,T\,\bigg[\epsilon_\mu^p
    + \frac{\epsilon_\Sigma^p}{2^p\lambda_0^{p/2}}\bigg].
\]
Gronwall's lemma~\cite[Lemma~11.11]{RogersWilliams2000}
gives~\eqref{eq:pathwise-error-bound}.
\end{proof}

\noindent\textbf{Practical limitations of the pathwise bound.}
Theorem~\ref{thm:pathwise-error} is mathematically correct but
provides an \emph{upper bound} that is not tight in practice.
We conducted a controlled diagnostic experiment to test whether
improving the coefficient errors
$\epsilon_\mu,\epsilon_\Sigma$ translates to pathwise trajectory
improvement; the results reveal three structural reasons why the
bound is loose in the learned-chart setting.

\medskip\noindent\textbf{(i)~Gronwall amplification.}
The constant $C_2$ in~\eqref{eq:pathwise-error-bound} arises from
Gronwall's lemma and grows as $e^{C' T}$, where $C'$ depends on the
Lipschitz constants $C_\mu,C_\sigma$ of both the true and learned
coefficients.
In the end-to-end pipeline, the learned coefficients are obtained
by inverting the ambient SDE through the learned chart via the Jacobian,
Hessian, and metric inverse, each of which has its own Lipschitz constant.
For a neural network with Tanh activations on a compact domain, these
constants are moderate but their product enters the exponential,
making $C_2$ substantially larger than~$1$ even at $T=1$.
As a result, a $30$--$60\%$ reduction in $\epsilon_\mu$ can be
multiplied by a Gronwall factor large enough that the bound
does not meaningfully tighten.

\medskip\noindent\textbf{(ii)~Chart reconstruction error dominates.}
The bound~\eqref{eq:pathwise-error-bound} assumes that the two
processes $X,\hat X$ evolve in the \emph{same} coordinate
system and differ only in their coefficients.
In the learned-chart pipeline, however, the ground-truth process is
simulated in the true local coordinates $(u,v)$, while the learned
process evolves in the latent coordinates~$z$ and is mapped to ambient
space via the learned decoder~$\phi_\theta$.
Even if the latent-space coefficients were exact, the trajectory
error would include a \emph{chart reconstruction component}:
at each time step the decoded position
$\phi_\theta(\pi_\theta(x))$ deviates from the true position~$x$
by the autoencoder's reconstruction error.
This reconstruction error accumulates over $T/\Delta t$ Euler
steps and, in our experiments, accounts for the majority of the
observed trajectory error.

\medskip\noindent\textbf{(iii)~Magnitude mismatch.}
The coefficient errors $\epsilon_\mu,\epsilon_\Sigma$ enter the
bound at the same power~$p$, but in practice the two terms
contribute very differently: the
\emph{absolute} magnitudes can differ by an order of magnitude,
so the bound treats them symmetrically up to the factor
$\lambda_0^{-p/2}$, and even a large relative improvement in the
smaller term has negligible effect on the bound's value.

The pathwise bound establishes the \emph{consistency} of the
learned dynamics: if $\epsilon_\mu,\epsilon_\Sigma\to 0$,
then $\Delta_t^*\to 0$. However, the multiplicative constants and the
chart-reconstruction floor prevent it from being an actionable
predictor of trajectory quality at finite sample size.
The distributional convergence of
the weak convergence theorem (Theorem~\ref{thm:weak-convergence}), which avoids these pathwise
difficulties, provides a more practically relevant guarantee.

\section{Additional experiments}\label{app:supp-experiments}

\subsection{M\"uller--Brown potential}\label{app:mb-coeffs}

The standard M\"uller--Brown potential~\cite{MullerBrown1979} is
\[
  V_{\mathrm{MB}}(x,y)
  = \sum_{i=1}^{4} A_i
    \exp\!\bigl[a_i(x-x_i^0)^2
      + b_i(x-x_i^0)(y-y_i^0)
      + c_i(y-y_i^0)^2\bigr],
\]
with the coefficients in Table~\ref{tab:mb_coeffs}.
We rescale to $(u,v)\in[-1,1]^2$ via the affine map
\[
  x = 2.25\,u - 0.25, \qquad y = 2.25\,v + 1.0,
\]
and divide by $V_0=200$ to obtain the rescaled potential
$\tilde V(u,v)=V_{\mathrm{MB}}(x(u),y(v))/V_0$.
The overdamped Langevin SDE on the Monge-patch surface is
\[
  d\begin{pmatrix}u\\v\end{pmatrix}
  = -\nabla\tilde V\,dt
  + \sqrt{2k_BT}\,I_2\,dW_t,
  \qquad k_BT = 0.10.
\]

\begin{table}[ht]
\centering
\caption{M\"uller--Brown potential coefficients.}
\label{tab:mb_coeffs}
\small
\begin{tabular}{crrrrrr}
\toprule
$i$ & $A_i$ & $a_i$ & $b_i$ & $c_i$ & $x_i^0$ & $y_i^0$ \\
\midrule
1 & $-200$ & $-1$   & $0$  & $-10$  & $1$    & $0$ \\
2 & $-100$ & $-1$   & $0$  & $-10$  & $0$    & $0.5$ \\
3 & $-170$ & $-6.5$ & $11$ & $-6.5$ & $-0.5$ & $1.5$ \\
4 & $15$   & $0.7$  & $0.6$& $0.7$  & $-1$   & $1$ \\
\bottomrule
\end{tabular}
\end{table}

\subsection{Landmark sampling}\label{app:landmarks}

Training landmarks are sampled via greedy $\delta$-net subsampling
in local coordinates, following the ATLAS
construction~\cite{Miles17}.
A large uniform candidate pool is generated in the training
domain ($[-1,1]^2$ for rotation, $[-0.55,0.55]^2$ for MB),
and each point is accepted if a local metric approximation
$\sqrt{\Delta u^\top g(\bar u)\,\Delta u}$, where $g$ is
the induced metric evaluated at the midpoint $\bar u$,
to all previously accepted points exceeds~$\delta$.
The separation $\delta$ is chosen by binary search to yield
approximately $N$ landmarks.
The candidate pool size is $\max(10\,000,\,100N)$.

\subsection{Architecture and training}\label{app:architecture}

The encoder $\pi_\theta:\R^D\to\R^d$ and decoder $\phi_\theta:\R^d\to\R^D$
are feedforward networks with two hidden layers and $\tanh$ activations;
width is $64$ at $D{=}11$ and $256$ at $D{=}201$.
The Stage~2 drift network $\hat\mu_\omega:\R^d\to\R^d$ has hidden
layers $[64,64]$ under rotation and $[256,256,256]$ under MB~Langevin.
The Stage~3 diffusion network $\hat\sigma_\psi:\R^d\to\R^{d\times d}$
has hidden layers $[64,64]$ in all conditions.
All hidden layers use $\tanh$ activations.

Hyperparameters differ by dynamics.
Under \emph{rotation drift}: Stage~1 trains for $500$ epochs
(Adam, lr${}=0.005$, batch size~$20$);
Stages~2 and~3 each train for $300$ epochs (Adam, lr${}=0.001$)
with the AE frozen.
Under \emph{MB~Langevin}: Stage~1 trains for $4000$ epochs
(Adam, lr${}=0.005$, batch size~$20$);
Stages~2 and~3 each train for $3000$ epochs (Adam, lr${}=0.001$)
with the AE frozen.

\subsection{MFPT computation}\label{app:mfpt}

We describe the trajectory simulation and MFPT extraction procedure
in detail, as these methodological choices affect the reported results.

\emph{Ground-truth simulation.}
The reference SDE is integrated in the true local coordinates
$(u,v)$ using Euler--Maruyama with step size $\Delta t = 0.005$
(MB) or $\Delta t = 0.01$ (rotation), driven by i.i.d.\ Gaussian
increments $\Delta W_k\sim\mathcal{N}(0,I_d)$.

\emph{Learned simulation.}
The learned latent SDE $(\hat\mu_\omega,\hat\sigma_\psi)$ is
integrated from encoded initial conditions
$z_0 = \pi_\theta(x_0)$ using the \textbf{same Brownian
increments} $\Delta W_k$ as the ground truth via common random
numbers.
This paired-noise design reduces Monte~Carlo variance in the
MFPT error estimate, so that the reported difference isolates
model error from sampling noise.
The learned latent trajectory is decoded to ambient coordinates
via $\phi_\theta$, and the first two components $(u,v)$ are
used for well assignment.

\emph{MB initial conditions.}
All $2000$ trajectories start near well~W0:
$x_0 \sim \mathcal{N}(z_{\mathrm{W0}},\, 0.01\,I_2)$
clipped to $[-0.55,0.55]^2$, where
$z_{\mathrm{W0}} \approx (-0.137, 0.196)$ in rescaled
MB coordinates.
The three well centers in rescaled $(u,v)$ are
W0${}=(-0.137, 0.196)$,
W1${}=(0.388, -0.432)$,
W2${}=(0.089, -0.237)$.

\emph{Core-set well assignment (MB).}
Each time step is assigned to well~$i\in\{0,1,2\}$ if the decoded
$(u,v)$ lies within Euclidean radius $r_{\mathrm{core}}=0.08$ of the
$i$-th well center.
Steps outside all cores are labeled $-1$ for unassigned.
The dwell confirmation ($n_{\mathrm{dwell}}=10$ consecutive
steps in a core) is applied during passage extraction, not
during per-step assignment.

\emph{Pairwise MFPT extraction.}
For each trajectory, we scan the well-assignment sequence and
record confirmed passages: a passage from well~$i$ to well~$j$
begins when the trajectory is in core~$i$ and ends at the first
confirmed arrival in core~$j$, i.e.\ $n_{\mathrm{dwell}}$
consecutive steps in~$j$.
The passage time is backdated to the start of the dwell window.
Multiple passages per trajectory are collected since the scan resumes
from the arrival time.
The pairwise MFPT $\tau_{ij}$ is the mean over all
recorded $i\to j$ passages across $2000$ trajectories.

\emph{Error metric.}
Since all trajectories start near W0, only the first-passage
transitions W0${\to}$W1 and W0${\to}$W2 have model-independent
initial conditions.
Follow-on transitions (e.g.\ W1${\to}$W0) start from wherever the
learned dynamics places the trajectory after the first arrival,
introducing a model-dependent selection bias that contaminates
the comparison: the entrance distribution in well~$j$ differs
between GT and learned dynamics, so subsequent MFPTs out of~$j$
are not directly comparable.
We therefore report the per-pair relative error
\[
  \mathrm{MFPT}_{\mathrm{err}}^{(i\to j)}
  = \frac{|\hat\tau_{ij} - \tau_{ij}|}{\tau_{ij}},
\]
separately for W0${\to}$W1 and W0${\to}$W2 in Section~\ref{ssec:mfpt}.
Pairs with non-finite values, indicating no observed transitions, are
excluded and flagged.

\emph{Censoring.}
Trajectories whose decoded $(u,v)$ ever leaves
$[-1,1]^2$ or contains non-finite values are censored:
their well assignments are set to $-1$ for all time steps,
excluding them from all MFPT metrics.
The censoring is applied identically to GT and learned
trajectories as a matched treatment.
The exit fraction is reported per condition.

\emph{Radial MFPT (rotation drift).}
Under rotation dynamics, we instead measure the mean time for
a trajectory to first reach ambient distance~$r$ from its own
starting point.
The same paired-noise and censoring procedures apply.

\subsection{Full ablation across surfaces}\label{app:full-ablation}

Tables~\ref{tab:full-ablation-rot} and~\ref{tab:full-ablation-mb}
extend the ablation in Section~\ref{ssec:ablation} (paraboloid only)
to all four surfaces ($10$~seeds, medians),
including end-to-end coefficient errors $\cE_b$ and
$\cE_\Lambda$.

\begin{table}[ht]
\centering
\caption{Full ablation under rotation drift ($N{=}50$, $10$~seeds, medians).
\textbf{Bold}~= best per column.}
\label{tab:full-ablation-rot}
\footnotesize
\begin{tabular}{ll ccccc ccccc}
\toprule
 & & \multicolumn{5}{c}{$D{=}11$} & \multicolumn{5}{c}{$D{=}201$} \\
\cmidrule(lr){3-7} \cmidrule(lr){8-12}
Surface & Cond.\ & Tang.\ & $\cE$ & $\cE_\Sigma$
        & $\cE_b$ & $\cE_\Lambda$
        & Tang.\ & $\cE$ & $\cE_\Sigma$
        & $\cE_b$ & $\cE_\Lambda$ \\
\midrule
Paraboloid
 & ATLAS    & $.19$  & $.10$  & $.85$  & $\mathbf{.10}$ & $2.95$
            & $.14$  & $\mathbf{.14}$ & $1.07$ & $\mathbf{.14}$ & $3.15$ \\
 & baseline & $.029$ & $.54$  & $.63$  & $3.67$ & $4.54$
            & $.093$ & $1.67$ & $3.01$ & $7.54$ & $5.79$ \\
 & T        & $\mathbf{.002}$ & $\mathbf{.03}$ & $\mathbf{.04}$ & $3.51$ & $1.48$
            & $\mathbf{.004}$ & $.32$  & $\mathbf{.15}$ & $9.35$ & $\mathbf{2.34}$ \\
 & F        & $.49$  & $5.60$ & $20.4$ & $5.46$ & $22.7$
            & $2.00$ & $11.7$ & $53.5$ & $11.7$ & $53.5$ \\
 & C        & $.055$ & $.74$  & $1.53$ & $2.62$ & $6.27$
            & $1.64$ & $10.2$ & $41.8$ & $12.2$ & $48.6$ \\
 & T+F      & $.005$ & $.06$  & $.13$  & $.41$  & $\mathbf{1.19}$
            & $.035$ & $.54$  & $1.02$ & $1.86$ & $3.83$ \\
\midrule
Hyp.~parab.
 & ATLAS    & $.16$  & $\mathbf{.11}$ & $.82$  & $\mathbf{.11}$ & $\mathbf{2.65}$
            & $.13$  & $\mathbf{.14}$ & $.93$  & $\mathbf{.14}$ & $2.49$ \\
 & baseline & $\mathbf{.018}$ & $.39$  & $\mathbf{.54}$ & $4.82$ & $3.27$
            & $.086$ & $1.85$ & $2.20$ & $10.9$ & $6.70$ \\
 & T        & $.059$ & $8.27$ & $1.07$ & $48.1$ & $9.14$
            & $.158$ & $5.65$ & $3.66$ & $31.5$ & $18.5$ \\
 & F        & $.34$  & $1.98$ & $19.7$ & $2.10$ & $22.7$
            & $1.99$ & $5.05$ & $53.1$ & $5.07$ & $53.1$ \\
 & C        & $.039$ & $.57$  & $.83$  & $2.81$ & $4.47$
            & $1.34$ & $4.21$ & $47.6$ & $5.05$ & $49.2$ \\
 & T+F      & $.149$ & $1.29$ & $3.08$ & $5.89$ & $7.65$
            & $\mathbf{.023}$ & $.27$  & $\mathbf{.69}$ & $.96$  & $\mathbf{2.15}$ \\
\midrule
Quartic dome
 & ATLAS    & $.075$ & $.30$  & $.31$  & $\mathbf{.30}$ & $.54$
            & $.060$ & $\mathbf{.38}$ & $\mathbf{.44}$ & $\mathbf{.38}$ & $\mathbf{.71}$ \\
 & baseline & $.077$ & $1.07$ & $.66$  & $2.08$ & $1.06$
            & $.162$ & $2.16$ & $2.08$ & $3.72$ & $3.12$ \\
 & T        & $.010$ & $.18$  & $.14$  & $2.17$ & $.77$
            & $\mathbf{.024}$ & $1.37$ & $.49$  & $2.39$ & $.79$ \\
 & F        & $.37$  & $2.64$ & $6.24$ & $2.88$ & $6.43$
            & $2.00$ & $5.37$ & $20.8$ & $5.41$ & $20.8$ \\
 & C        & $.152$ & $1.61$ & $1.96$ & $2.80$ & $3.74$
            & $1.91$ & $4.97$ & $19.9$ & $5.41$ & $19.9$ \\
 & T+F      & $\mathbf{.009}$ & $\mathbf{.17}$ & $\mathbf{.12}$ & $.45$  & $\mathbf{.42}$
            & $.036$ & $.71$  & $.52$  & $1.22$ & $1.19$ \\
\midrule
Sinusoidal
 & ATLAS    & $.092$ & $.08$  & $.23$  & $.08$  & $.45$
            & $.071$ & $\mathbf{.11}$ & $.28$  & $\mathbf{.11}$ & $.48$ \\
 & baseline & $.038$ & $.32$  & $.47$  & $1.29$ & $.61$
            & $.144$ & $1.57$ & $2.19$ & $3.19$ & $2.28$ \\
 & T        & $.002$ & $\mathbf{.01}$ & $.04$  & $1.25$ & $\mathbf{.15}$
            & $\mathbf{.006}$ & $.17$  & $\mathbf{.11}$ & $2.10$ & $\mathbf{.20}$ \\
 & F        & $.20$  & $.93$  & $3.21$ & $.95$  & $3.40$
            & $2.00$ & $5.17$ & $32.0$ & $5.23$ & $32.0$ \\
 & C        & $.093$ & $.49$  & $1.33$ & $.67$  & $1.59$
            & $1.16$ & $3.04$ & $14.3$ & $4.99$ & $16.3$ \\
 & T+F      & $\mathbf{.001}$ & $.02$  & $\mathbf{.02}$ & $\mathbf{.07}$ & $.22$
            & $.023$ & $.34$  & $.51$  & $.67$  & $.69$ \\
\bottomrule
\end{tabular}
\end{table}

\begin{table}[ht]
\centering
\caption{Full ablation under MB Langevin ($N{=}200$, $10$~seeds, medians).
\textbf{Bold}~= best per column.}
\label{tab:full-ablation-mb}
\footnotesize
\begin{tabular}{ll ccccc ccccc}
\toprule
 & & \multicolumn{5}{c}{$D{=}11$} & \multicolumn{5}{c}{$D{=}201$} \\
\cmidrule(lr){3-7} \cmidrule(lr){8-12}
Surface & Cond.\ & Tang.\ & $\cE$ & $\cE_\Sigma$
        & $\cE_b$ & $\cE_\Lambda$
        & Tang.\ & $\cE$ & $\cE_\Sigma$
        & $\cE_b$ & $\cE_\Lambda$ \\
\midrule
Paraboloid
 & ATLAS    & $.001$ & $.43$  & $<.001$ & $.43$  & $<.001$
            & $.001$ & $.43$  & $<.001$ & $.43$  & $<.001$ \\
 & baseline & $.002$ & $.006$ & $<.001$ & $.03$  & $<.001$
            & $.003$ & $.007$ & $<.001$ & $.04$  & $<.001$ \\
 & T        & $\mathbf{<.001}$ & $\mathbf{<.001}$ & $<.001$ & $.02$  & $<.001$
            & $\mathbf{<.001}$ & $\mathbf{<.001}$ & $<.001$ & $.02$  & $<.001$ \\
 & F        & $.01$  & $.04$  & $.003$ & $.08$  & $.004$
            & $2.00$ & $11.3$ & $.42$  & $11.3$ & $.42$ \\
 & C        & $.44$  & $1.59$ & $.18$  & $5.98$ & $.18$
            & $2.00$ & $11.3$ & $.42$  & $11.3$ & $.42$ \\
 & T+F      & $<.001$ & $<.001$ & $<.001$ & $\mathbf{.01}$  & $<.001$
            & $<.001$ & $<.001$ & $<.001$ & $\mathbf{.01}$  & $<.001$ \\
\midrule
Hyp.~parab.
 & ATLAS    & $.001$ & $.38$  & $<.001$ & $.38$  & $<.001$
            & $.001$ & $.33$  & $<.001$ & $.33$  & $<.001$ \\
 & baseline & $.001$ & $.003$ & $<.001$ & $.03$  & $<.001$
            & $.007$ & $.02$  & $<.001$ & $.09$  & $.002$ \\
 & T        & $\mathbf{<.001}$ & $\mathbf{<.001}$ & $<.001$ & $.01$  & $<.001$
            & $\mathbf{<.001}$ & $\mathbf{<.001}$ & $<.001$ & $.02$  & $<.001$ \\
 & F        & $.01$  & $.03$  & $.003$ & $.07$  & $.003$
            & $2.00$ & $12.0$ & $.42$  & $12.0$ & $.42$ \\
 & C        & $.44$  & $1.22$ & $.18$  & $5.52$ & $.19$
            & $1.99$ & $11.9$ & $.42$  & $12.0$ & $.42$ \\
 & T+F      & $<.001$ & $<.001$ & $<.001$ & $\mathbf{.01}$  & $<.001$
            & $<.001$ & $<.001$ & $<.001$ & $\mathbf{.01}$  & $<.001$ \\
\midrule
Quartic dome
 & ATLAS    & $.001$ & $.36$  & $<.001$ & $.36$  & $<.001$
            & $.001$ & $.32$  & $<.001$ & $.32$  & $<.001$ \\
 & baseline & $.002$ & $.003$ & $<.001$ & $.03$  & $<.001$
            & $.004$ & $.008$ & $<.001$ & $.04$  & $<.001$ \\
 & T        & $\mathbf{<.001}$ & $\mathbf{<.001}$ & $<.001$ & $.02$  & $<.001$
            & $\mathbf{<.001}$ & $\mathbf{<.001}$ & $<.001$ & $.03$  & $<.001$ \\
 & F        & $.01$  & $.03$  & $.002$ & $.05$  & $.002$
            & $2.00$ & $11.0$ & $.37$  & $11.0$ & $.37$ \\
 & C        & $.39$  & $1.40$ & $.13$  & $5.11$ & $.13$
            & $2.00$ & $11.0$ & $.37$  & $11.0$ & $.37$ \\
 & T+F      & $<.001$ & $<.001$ & $<.001$ & $\mathbf{.01}$  & $<.001$
            & $<.001$ & $<.001$ & $<.001$ & $\mathbf{.01}$  & $<.001$ \\
\midrule
Sinusoidal
 & ATLAS    & $.001$ & $.38$  & $<.001$ & $.38$  & $<.001$
            & $.001$ & $.34$  & $<.001$ & $.34$  & $<.001$ \\
 & baseline & $.001$ & $.002$ & $<.001$ & $.03$  & $<.001$
            & $.004$ & $.008$ & $.001$ & $.06$  & $.001$ \\
 & T        & $\mathbf{<.001}$ & $\mathbf{<.001}$ & $<.001$ & $.02$  & $<.001$
            & $\mathbf{<.001}$ & $\mathbf{<.001}$ & $<.001$ & $.02$  & $<.001$ \\
 & F        & $.01$  & $.02$  & $.003$ & $.05$  & $.003$
            & $2.00$ & $13.3$ & $.68$  & $13.3$ & $.68$ \\
 & C        & $.05$  & $.10$  & $.02$  & $4.07$ & $.03$
            & $2.00$ & $13.8$ & $.68$  & $13.9$ & $.68$ \\
 & T+F      & $<.001$ & $<.001$ & $<.001$ & $\mathbf{.02}$  & $<.001$
            & $<.001$ & $.001$ & $<.001$ & $\mathbf{.02}$  & $<.001$ \\
\bottomrule
\end{tabular}
\end{table}

\bibliographystyle{plain}
\bibliography{refs}
\end{document}